\documentclass[letterpaper,11pt]{article}
\usepackage[toc,page]{appendix}
\usepackage[margin=1in]{geometry}
\usepackage[bookmarks, colorlinks=true, plainpages = false, citecolor = blue,linkcolor=red,urlcolor = blue, filecolor = blue,pagebackref]{hyperref}
\usepackage{url}
\usepackage{amsmath,amsfonts,amsthm,amssymb,bm, verbatim,dsfont,mathtools}
\usepackage{color,graphicx,appendix}
\usepackage{subfigure}
\usepackage{etoolbox}
\usepackage{tikz}
\usepackage{xr,xspace}
\usepackage{todonotes}
\usepackage{paralist}
\usepackage{caption,soul}
\usepackage{algorithm}
\usepackage{algorithmic}
\makeatletter

\newtheorem{theorem}{Theorem}
\newtheorem{lemma}{Lemma}

\newtheorem{corollary}{Corollary}
\theoremstyle{definition}
\newtheorem{definition}{Definition}

\newtheorem{remark}{Remark}
\newtheorem{assumption}{Assumption}

\renewcommand{\check}[1]{\widetilde{#1}}

\renewcommand{\implies}{\Rightarrow}

\newcommand{\argmax}{\mathop{\arg\max}}

\usepackage{xspace,prettyref}
\newcommand{\CML}{\widehat{C}_{\rm ML}}
\newcommand{\diverge}{\to\infty}

\newcommand{\reals}{{\mathbb{R}}}

\newcommand{\naturals}{{\mathbb{N}}}

\newcommand{\supp}{{\rm supp}}
\newcommand{\eexp}{{\rm e}}

\newcommand{\diff}{{\rm d}}

\newcommand{\Expect}{\mathbb{E}}
\newcommand{\expect}[1]{\mathbb{E}\left[ #1 \right]}

\newcommand{\Prob}{\mathbb{P}}
\newcommand{\pprob}[1]{ \mathbb{P}\{ #1 \} }
\newcommand{\prob}[1]{ \mathbb{P}\left\{ #1 \right\} }

\newcommand{\var}{\mathsf{var}}

\newcommand{\Bern}{{\rm Bern}}
\newcommand{\Binom}{{\rm Binom}}

\newcommand{\eg}{e.g.\xspace}
\newcommand{\ie}{i.e.\xspace}

\newrefformat{eq}{(\ref{#1})}
\newrefformat{chap}{Chapter~\ref{#1}}
\newrefformat{sec}{Section~\ref{#1}}
\newrefformat{alg}{Algorithm~\ref{#1}}
\newrefformat{fig}{Fig.~\ref{#1}}
\newrefformat{tab}{Table~\ref{#1}}
\newrefformat{rmk}{Remark~\ref{#1}}
\newrefformat{clm}{Claim~\ref{#1}}
\newrefformat{def}{Definition~\ref{#1}}
\newrefformat{cor}{Corollary~\ref{#1}}
\newrefformat{lmm}{Lemma~\ref{#1}}
\newrefformat{prop}{Proposition~\ref{#1}}
\newrefformat{app}{Appendix~\ref{#1}}
\newrefformat{hyp}{Hypothesis~\ref{#1}}
\newrefformat{thm}{Theorem~\ref{#1}}
\newrefformat{ass}{Assumption~\ref{#1}}

\newcommand{\pth}[1]{\left( #1 \right)}
\newcommand{\qth}[1]{\left[ #1 \right]}
\newcommand{\sth}[1]{\left\{ #1 \right\}}

\newcommand{\iprod}[2]{\left \langle #1, #2 \right\rangle}

\newcommand{\indc}[1]{{\mathbf{1}_{\left\{{#1}\right\}}}}

\newcommand{\calE}{{\mathcal{E}}}

\newcommand{\calN}{{\mathcal{N}}}

\newcommand{\ML}{{\rm ML}\xspace}

\newcommand{\ER}{Erd\H{o}s-R\'enyi\xspace}

\renewcommand{\hat}{\widehat}
\renewcommand{\tilde}{\widetilde}

\begin{document}

\title{Information Limits for Recovering a Hidden Community}

\date{\today}
\author{ 
Bruce Hajek \and Yihong Wu \and Jiaming Xu\thanks{
B. Hajek and Y. Wu are with
the Department of ECE and Coordinated Science Lab, University of Illinois at Urbana-Champaign, Urbana, IL, \texttt{\{b-hajek,yihongwu\}@illinois.edu}.
J. Xu is with the Simons Institute for the Theory of Computing, University of California, Berkeley, Berkeley, CA, 
\texttt{jiamingxu@berkeley.edu}.}
}

\maketitle

\begin{abstract}
We study the problem of recovering a hidden community of cardinality $K$ from an $n \times n$
symmetric data matrix $A$, where for distinct indices $i,j$, $A_{ij} \sim P$ if $i, j$ both belong to 
the community and $A_{ij} \sim Q$ otherwise, for two known probability distributions
$P$ and $Q$ depending on $n$.  If $P={\rm Bern}(p)$ and $Q={\rm Bern}(q)$ with $p>q$, it reduces to the problem of
finding a densely-connected $K$-subgraph planted in a large Erd\"os-R\'enyi graph;
if $P=\mathcal{N}(\mu,1)$ and $Q=\mathcal{N}(0,1)$ with $\mu>0$, it corresponds to the problem of
locating a $K \times K$ principal submatrix of elevated means in a large Gaussian random matrix.
We focus on two types of asymptotic recovery guarantees as $n \to \infty$: (1) weak recovery: expected
number of classification errors  is $o(K)$;  (2) exact recovery:  probability of classifying all indices
correctly converges to one.  Under mild assumptions on $P$ and $Q$, and allowing the community
size to scale sublinearly with $n$, we derive a  set of  sufficient conditions  and a set of necessary
conditions for recovery,   which are asymptotically tight with sharp constants.
The results hold in particular for the Gaussian case, and for the case of bounded log
likelihood ratio, including the Bernoulli case whenever $\frac{p}{q}$  and $\frac{1-p}{1-q}$ are bounded away from zero and infinity.
  An important algorithmic implication is that, whenever exact recovery
is information theoretically possible,  any algorithm that provides weak recovery when the
community size is concentrated near $K$  can be upgraded to achieve exact recovery in linear
additional time by a simple voting procedure. 
\end{abstract}

\section{Introduction}
Many modern datasets can be represented as networks with vertices denoting the
objects and edges 
(sometimes weighted or labeled) 
encoding their pairwise interactions.
An interesting problem is  to identify a group of vertices  with atypical  interactions.
In social network analysis, this group can be interpreted as a community
with higher edge connectivities than the rest of the network;
in microarray experiments,
 this group may correspond to a set of differentially expressed genes.
 To study this problem, we
investigate the following probabilistic model considered in \cite{Deshpande12}.

\begin{definition}[Hidden Community Model]  \label{def:model}
 Let $C^*$ be drawn uniformly at random from all subsets of $[n]$ of cardinality $K$.
 Given probability measures $P$ and $Q$ on a common measurable space, let $A$ be an $n \times n$ symmetric matrix with zero diagonal
 where for all $1 \le i<j \le n$, $A_{ij}$ are mutually independent, and $A_{ij} \sim P$ if $i,j\in C^*$ and $A_{ij} \sim Q$ otherwise.
\end{definition}

In this paper we assume that we only have access to pairwise information $A_{ij}$ for distinct indices $i$ and $j$ whose distribution is either $P$ or $Q$ depending on the community membership;
no direct observation about the individual indices is available (hence the zero diagonal of $A$).
Two choices of $P$ and $Q$ arising in many applications are the following:
\begin{itemize}
\item {Bernoulli case}: 
 $P=\Bern(p)$ and $Q=\Bern(q)$ with $p\neq q $. When $p>q$, this coincides with the {\em planted dense subgraph model}
studied in \cite{McSherry01, arias2013community, ChenXu14, HajekWuXu14,Montanari:15OneComm},
which is also a special case of the general stochastic block model~\cite{Holland83} with a single community.
In this case, the data matrix $A$ corresponds to the adjacency matrix of a graph, where two vertices are connected with probability $p$ if both belong to the community $C^\ast$, and with probability $q$ otherwise.
Since $p > q$, the subgraph induced by $C^\ast$ is likely to be denser than the rest
of the graph. 

\item {Gaussian case}: $P=\calN(\mu,1)$ and $Q=\calN(0,1)$ with $\mu \neq 0$. 
This corresponds to a symmetric version of the {\em submatrix localization} problem studied in \cite{shabalin2009submatrix,kolar2011submatrix,butucea2013,Butucea2013sharp,ma2013submatrix,ChenXu14,CLR15}.\footnote{The previously studied submatrix localization model (also known as noisy biclustering) deals with submatrices whose row and column supports need not coincide and the noise matrix is asymmetric consisting of iid entries throughout. Here we focus on locating principal submatrices contaminated by a symmetric noise matrix. Additionally, we assume the diagonal does not carry any information. If instead we assume nonzero diagonal with $A_{ii} \sim \calN(\mu,1)$ if $ i \in C^*$ and $A_{ii} \sim \calN(0,1)$ if $i \notin C^*$, the results in this paper carry over with minor modifications explained in \prettyref{rmk:nonzero_diagonal}.}
When $\mu>0$, the entries of $A$ with row and column indices in $C^\ast$ have positive mean $\mu$ except those on the diagonal, while the rest of the entries have zero mean. 

\end{itemize}

Given the data matrix $A$, the problem of interest is to accurately recover the underlying community $C^\ast$.
The distributions $P$ and $Q$ as well as the community size $K$ depend on the matrix size $n$ in general.
For simplicity we assume that these model parameters 
are known to the  estimator.
The only assumptions on the community size $K$ we impose are that $K/n$ is bounded away from one,
and, to avoid triviality, that $K\geq 2$.   Of particular interest is the case of  $K=o(n),$
where the community size grows sublinearly.

We focus on the following two types of recovery guarantees.\footnote{Exact and weak recovery are
called strong consistency and weak consistency in \cite{Mossel14}, respectively.  
}
Let $\xi \in  \{0,1\}^n$ denote the indicator of the community such that $\supp(\xi)=C^*$. 
Let $\hat\xi=\hat\xi(A) \in \{0,1\}^n$ be  an estimator.

\begin{definition}[Exact Recovery]
\label{def:exact}
Estimator $\hat \xi$ {\em exactly recovers} $\xi,$ if, 
as $n \to \infty$,
$
\Prob[\xi \neq \hat\xi] \to 0,
$
where the probability is with respect to the randomness of $\xi$ and $A$.
\end{definition}

\begin{definition}[Weak Recovery]   \label{def:weak_recovery}
Estimator $\hat \xi$ {\em weakly recovers} $\xi$ if, as
$n \to \infty$,  $d_H(\xi, \hat\xi) / K \to 0$ in probability,  where $d_H$ denotes the Hamming distance.
\end{definition}

The existence of an estimator satisfying \prettyref{def:weak_recovery} is equivalent
 to the existence of an estimator such that  $ \Expect[d_H(\xi, \hat\xi)] = o(K)$  (see \prettyref{app:weak} for a proof).
Clearly, any estimator achieving exact recovery also achieves weak recovery; for bounded $K$,  exact and
weak recovery are  equivalent.

Intuitively, for a fixed network size $n$, as the community size $K$ decreases, or the distributions
$P$ and $Q$ get closer together, the recovery problem becomes harder. In this paper, we aim to
address the following question:
\emph{From an information-theoretic perspective, computational considerations aside, what are the fundamental limits of recovering the community?}
Specifically, we derive sharp necessary and sufficient conditions 
in terms of the model parameters 
under which the community can be 
exactly or weakly recovered. These results serve as benchmarks for evaluating practical algorithms and aid us in understanding the performance limits of polynomial-time algorithms.

In addition to establishing information limits with sharp constants for general $P$ and $Q$,
we identify  the following algorithmic connection between weak and exact recovery:
 If exact recovery is information-theoretically  possible and there is an algorithm
for weak recovery, then in linear additional time we can obtain exact recovery based on the weak recovery algorithm.
This suggests that if the information limit of weak recovery can be obtained in polynomial time, then so can exact recovery;
conversely, if there exists a computational barrier that separates the information limit and the performance of polynomial-time algorithms for exact recovery,
then weak recovery also suffers from such a barrier.
To establish the connection, we apply a two-step procedure: the first step uses
an estimator capable of weak recovery, even in the presence of a slight mismatch between $|C^\ast|$ and $K$, such as the maximum likelihood estimator (see \prettyref{lmm:weak_general_random_suff}); the second step cleans up the residual errors through a local voting
procedure for each index. In order to ensure the first and second step are independent, 
we use a method which we call {\em successive withholding}.
The method of successive withholding is to randomly partition
the set of indices into a finite number of subsets.   One at a time,
one subset is withheld to produce a reduced set of indices,  and an
estimation algorithm is run on the reduced set of indices.   The
estimate obtained from the reduced set of indices is used to classify the
indices in the withheld subset.   The idea is to gain independence:
the outcome of estimation based on the reduced set of indices is
independent of the data between the
withheld indices and the reduced set of indices, and the withheld subset is sufficiently small so that we can still obtain sharp constants.
This method is mentioned in \cite{Condon01},
and variations of it have been used in  \cite{Condon01}, 
\cite{MosselNeemanSlyCOLT14}, and \cite{Mossel14}.

\subsection{Related Work}
	\label{sec:related}

Previous work has determined the information limits for
exact recovery up to universal constant factors for some choices
of $P$ and $Q$.
For the Bernoulli case, it is shown in \cite{ChenXu14} that
if $K d(q\|p) - c  \log K \to \infty $ and $ K d(p\|q) \ge c \log n$ for some large constant $c>0$, then exact recovery
is achievable via the maximum likelihood estimator (MLE);
conversely, if $K d(q\|p) \le c' \log K  $ and $K d(p\|q) \le  c' \log n$ for some small constant $c'>0$, then exact recovery is impossible for any algorithms.
Similarly, for the Gaussian case, it is proved in \cite{kolar2011submatrix} that if $K \mu^2 \ge c \log n$, then exact recovery is achievable via the MLE;
conversely, if $K \mu^2 \le c' \log n$, exact recovery is impossible for any algorithms.  
To the best of our knowledge, there are only a  few special cases where the information limits with \emph{sharp} constants are known:
\begin{itemize}
\item Bernoulli case with $p=1$ and $q=1/2$: It is widely known as the planted clique problem \cite{Jer92}.
If $ K\ge 2(1+\epsilon) \log_2 n $ for any $\epsilon>0$, exact recovery is achievable via the MLE;
if $ K \le 2(1-\epsilon) \log_2 n $, then exact recovery is impossible. Despite an extensive research effort
 polynomial-time algorithms are only known to achieve exact recovery for $K \ge c \sqrt{n} $ for any constant $c>0$~\cite{Alon98,Feige10findinghidden,Dekel10,Ames2011clique,Deshpande12}.
\item Bernoulli case with $p=a \log n /n$ and $q=b\log n/n$ for fixed $a,b$ and $K=\rho n$ for a fixed constant $0<\rho<1$.
The recent work \cite{HajekWuXuSDP14} finds an explicit threshold $\rho^*(a,b)$, such that
if $\rho > \rho^*(a,b)$, exact recovery is achievable in polynomial-time via
semi-definite relaxations of the MLE with probability tending to one; if $\rho < \rho^*(a,b)$, any estimator fails to exactly recover
the cluster with probability tending to one regardless of the computational costs.  This conclusion is in sharp
contrast to the computational barriers observed in the planted clique problem.

\item The paper of Butucea et al. \cite{Butucea2013sharp} gives sharp results for
a Gaussian submatrix recovery problem similar to
the one considered here -- see \prettyref{rmk:Butucea_sharp} for details.
\end{itemize}

While this paper focuses on information-theoretic limits, it complements other work
investigating computationally efficient recovery procedures, such as convex relaxations~\cite{ames2012clustering,Ames2013,ChenXu14,HajekWuXuSDP14,HajekWuXu_one_sdp15},
spectral methods~\cite{McSherry01}, and message-passing algorithms~\cite{Deshpande12,Montanari:15OneComm,HajekWuXu_MP_submat15,HajekWuXu_one_beyond_spectral15}. 
In particular, for both the Bernoulli and Gaussian cases:
\begin{itemize}
\item if $K=\Theta(n)$, a linear-time degree-thresholding algorithm achieves the information limit of weak recovery (see 
\cite[Appendix A]{HajekWuXu_one_beyond_spectral15} and \cite[Appendix A]{HajekWuXu_MP_submat15});
\item if $K=\omega(n/\log n)$, whenever information-theoretically possible, exact recovery can
be achieved in polynomial time using  semi-definite programming \cite{HajekWuXu_one_sdp15};
\item if $K \geq \frac{n}{\log n} (1/(8e) + o(1))$ for Gaussian case and $K \geq \frac{n}{\log n} (\rho_{\sf BP}(a/b) + o(1))$
for Bernoulli case,\footnote{Here $\rho_{\sf BP} (a/b)$ denotes a constant only depending on $a/b$.}
exact recovery can be attained in nearly linear time via message passing plus clean up~\cite{HajekWuXu_one_beyond_spectral15,HajekWuXu_MP_submat15} whenever information-theoretically possible.
\end{itemize}
However, it is an open problem
whether any polynomial time can achieve the respective information limit of weak recovery
for $K=o(n)$, or  exact recovery for $K \le \frac{n}{\log n} (1/(8e) -\epsilon)$ in the Gaussian case and 
for $K \le  \frac{n}{\log n} (\rho_{\sf BP}(a/b) -\epsilon)$ in the Bernoulli case, for any fixed $\epsilon>0$. \\

	The related work \cite{Montanari:15OneComm} studies weak recovery in the sparse regime of $p=a/n$, $q=b/n$,
 and $K=\kappa n$. In the iterated limit where first $n \to \infty$, and then $\kappa \to 0$ and $a, b\to \infty$,
with $\lambda=\frac{\kappa^2 (a-b)^2}{(1-\kappa) b }$ fixed, it is shown that a local algorithm, namely local belief propagation,
 achieves weak recovery in linear time if $\lambda \eexp >1$ and conversely, if $\lambda \eexp <1$, no local algorithm can achieve
 weak recovery.  
 Moreover, it is shown that for any $\lambda>0$, MLE achieves a recovery guarantee similar to weak recovery in \prettyref{def:weak_recovery}.
In comparison, the sharp information limit for weak recovery
 identified in \prettyref{cor:weak_Bern} below allows $p, q$ and $K$ to vary simultaneously with $n$ as $n\to \infty$.

Finally, we briefly compare the results of this paper to those of  \cite{Abbe14} and  \cite{Mossel14} on the planted bisection model (also known as the binary symmetric stochastic block model),  
 where the vertices are partitioned into two equal-sized communities.
  First, a necessary and sufficient condition for weak recovery and a necessary and sufficient condition for exact recovery
  are obtained in \cite{Mossel14}. In this paper, sufficient and
 necessary conditions,  \prettyref{eq:weak-bdd_suff}  and \prettyref{eq:weak-bdd_nec} in \prettyref{thm:weak_general}, 
are presented separately.  
These conditions match up except right at the boundary; we do not
determine whether recovery is possible exactly at the boundary.
 The result for exact recovery in \cite{Abbe14} is similar in that regard.   Perhaps
 future work, based on techniques from \cite{Mossel14}, can provide a more refined analysis for the recovery problem at the boundary. 
  Secondly, when 
recovery is information theoretically possible for the planted  bisection problem, efficient algorithms are shown
to exist in  \cite{Abbe14} and \cite{Mossel14}.   In contrast, for detecting or recovering a single community whose size is sublinear in the network size, there can be a significant gap between
what is information theoretically possible and what can be achieved by existing efficient algorithms (see \cite{Alon98,balakrishnan2011tradeoff,ma2013submatrix,HajekWuXu14,Montanari:15OneComm}).
We turn instead to the MLE for proof of optimal achievability.
Finally, this paper covers both the Gaussian and Bernoulli case
(and other distributions) in a unified framework without assuming that the community size scales linearly with the network size.

\paragraph{Notation}
For any positive integer $n$, let $[n]=\{1, \ldots, n\}$.
For any set $T \subset [n]$, let $|T|$ denote its cardinality and $T^c$ denote its complement.
We use standard big $O$ notations,
e.g., for any sequences $\{a_n\}$ and $\{b_n\}$, $a_n=\Theta(b_n)$ or $a_n  \asymp b_n$
if there is an absolute constant $c>0$ such that $1/c\le a_n/ b_n \le c$.
Let $\Binom(n,p)$ denote the binomial distribution with $n$ trials and success probability $p$.
Let $D(P\|Q)=\Expect_P[\log \frac{dP}{dQ}]$ denotes the Kullback-Leibler (KL) divergence between distributions $P$ and $Q$.
Let $\Bern(p)$ denote the Bernoulli distribution with mean $p$ and $d(p\|q) = D(\Bern(p)\|\Bern(q)) = p \log \frac{p}{q} +\bar p \log \frac{\bar p}{\bar q},$
where $\bar p \triangleq 1-p$.
Logarithms are natural and we adopt the convention $0 \log 0=0$. Let $\Phi(x)$ and $Q(x)$
denote the cumulative distribution function (CDF) and complementary CDF  of the standard normal distribution,
respectively.

\section{Overview of Main Results} \label{sec:overview}

\subsection{Background on Maximum Likelihood Estimator and Assumptions}

Given the data matrix $A$, a sufficient statistic for estimating the community $C^*$ is
the \emph{log likelihood ratio (LLR) matrix} $\bm{L} \in \reals^{n \times n}$, where
$L_{ij}=\log \frac{dP}{dQ}(A_{ij})$ for $i \neq j$ and $L_{ii}=0$.
For $S,T\subset [n]$,  define  
\begin{equation}
e(S,T) = \sum_{(i<j): (i,j) \in (S\times T) \cup (T\times S)} L_{ij}.
	\label{eq:eST}
\end{equation}
Let $\CML$ denote the maximum likelihood estimation (MLE) of $C^*,$ given by:
\begin{equation}
\CML=\argmax_{C\subset [n]} \{ e(C,C) : |C|= K \},
	\label{eq:MLE}
\end{equation}
which minimizes the error probability $\pprob{\widehat{C}\neq C^*}$ because $C^*$ is equiprobable by assumption.
Evaluating the MLE requires knowledge of $K$.
Computation of the MLE is NP hard for general values of $n$ and $K$ because
certifying the existence of a clique of a specified size in an
undirected graph, which is known to be an NP complete problem \cite{Karp72},
can be reduced to computation of the MLE.  Thus, evaluating the MLE in the
worst case is deemed computationally intractable. 
It is worth noting that the optimal estimator that minimizes the expected number of misclassified indices
(Hamming loss) is the bit-MAP decoder $\tilde\xi=(\tilde\xi_i)$, where 
$\tilde \xi_i \triangleq \argmax_{j\in\{0,1\}} \Prob[\xi_i=j|L]$. Therefore, although
the MLE is optimal for exact recovery, it need not be optimal for weak recovery;
nevertheless, we choose to analyze MLE due to its simplicity and it turns out to
be asymptotically optimal for weak recovery as well.

Our results require mild regularity conditions on the size of the hidden community $K$
and on the pair of distributions, $P$ and $Q.$  Specifically, for $K$,
{\em it is assumed without further comment that }
\[
	\limsup_{n \to \infty} K/n < 1.
\]
This assumption implies that $\frac{\log n}{\log (n-K)} \to 1$,  so in several
asymptotic results $\log n$ and $\log (n-K)$ are interchangeable; we give preference to $\log n$.
Also, to avoid triviality, {\em it is assumed throughout that $K\geq 2.$}

To state the assumption on $P$ and $Q$ we introduce some standard notation associated
with binary hypothesis testing based on independent samples. Throughout the paper we assume the KL divergences $D(P\|Q)$ and $D(Q\|P)$ are finite.  In particular, $P$ and $Q$ are mutually absolutely continuous, and the likelihood
ratio, $\frac{dP}{dQ},$  satisfies $\Expect_Q\left[ \frac{dP}{dQ} \right] = \Expect_P\left[ (\frac{dP}{dQ})^{-1} \right] =1$.
Let $L=\log \frac{dP}{dQ}$ denote the LLR.  The likelihood ratio test for
$n$ observations and threshold $n\theta$ is to declare $P$ to be the true distribution if
$\sum_{k=1}^{n} L_k \ge n\theta$ and to declare $Q$ otherwise.  
For $\theta \in [-D(Q\|P),D(P\|Q)]$,  the standard Chernoff bounds for error probability of this likelihood ratio test are given by:
\begin{align}
Q\qth{\sum_{k=1}^{n} L_k \ge n\theta} \leq \exp(-n E_Q(\theta)) \label{eq:LDupper1}
  \\
P\qth{\sum_{k=1}^{n} L_k \le n\theta} \leq \exp(-n E_P(\theta)) \label{eq:LDupper2},
\end{align}
where the log moment generating functions of $L$ are denoted by
 $\psi_Q(\lambda) = \log \Expect_Q[\exp(\lambda L)] $ and $\psi_P(\lambda) = \log \Expect_P[\exp(\lambda L)] 
= \psi_Q(\lambda + 1 )$ and the large deviations exponents are give by Legendre transforms of the
log moment generating functions:
\begin{equation}
E_Q(\theta) = \psi_Q^*(\theta) \triangleq \sup_{\lambda \in \reals} \lambda \theta - \psi_Q(\lambda), 
\quad E_P(\theta) =  \psi_P^*(\theta) \triangleq  \sup_{\lambda \in \reals} \lambda \theta - \psi_P(\lambda) =E_Q(\theta) -\theta.	
	\label{eq:ratefunction}
\end{equation}
In particular, $E_P$ and $E_Q$ are convex functions. Moreover, since $\psi_Q'(0)=-D(Q\|P)$ and $\psi_Q'(1)=D(P\|Q)$, we have $E_Q(-D(Q\|P)) = E_P(D(P\|Q)) = 0$ and hence 
$E_Q(D(P\|Q))=D(P\|Q)$ and $E_P(-D(Q\|P))=D(Q\|P)$. 
Our regularity assumption on the pair $P$ and $Q$ is the following.
\begin{assumption}  \label{ass:reg_strong}
 There exists a constant $C$ such that for all $n$, 
\begin{align}
 \psi_Q''(\lambda) \leq C  \min\{D(P\|Q),  D(Q\|P)\},  \quad \forall \lambda \in [-1,1].    \label{eq:reg_strong}
\end{align}
\end{assumption}
In general,  $\psi_Q''(\lambda) =\psi_P''(\lambda-1)  = \var_{Q_\lambda} (L) ,$  where
$Q_{\lambda}$ is the tilted distribution defined by $dQ_\lambda = \exp(\lambda L -\psi_Q(\lambda)) dQ,$
so the point of
\prettyref{ass:reg_strong}  is to require these quantities for $\lambda \in [-1,1]$  be bounded by a constant times
the  divergences.   \prettyref{ass:reg_strong} is the strongest condition imposed on $P$ and $Q$ in this
 paper; several of the results hold under weaker assumptions described in \prettyref{sec:technical}, which are also weaker than sub-Gaussianity of the LLR.

\prettyref{ass:reg_strong} is fulfilled in the following cases:
\begin{enumerate}
	\item Bounded LLR: \prettyref{lmm:scale} in \prettyref{sec:technical} shows that \prettyref{ass:reg_strong} holds if $L$ is bounded by a constant, which, in particular, holds in the Bernoulli case
if both $\frac{p}{q}$ and $\frac{\bar p}{\bar q}$ are bounded away from zero and infinity.

\item Gaussian case: In the Gaussian case $P=\calN(\mu,1), Q=\calN(0,1)$, we have $L(x)=\mu(x-\frac{\mu}{2})$,
$D(P\|Q)=D(Q\|P)=\mu^2/2$,  $\psi_Q(\lambda)=\frac{(\lambda^2-\lambda)\mu^2}{2}$,
$E_Q(\theta) = \frac{1}{8}(\mu+ \frac{2\theta}{\mu})^2$ and $E_P(\theta) =E_Q(-\theta)$. In particular,
$\psi''_Q(\lambda) \equiv \mu^2$ so \prettyref{ass:reg_strong} holds with $C=2$ regardless of how $\mu$ varies with $n$.
More generally, for $P$ and $Q$ lying in the same exponential family, \prettyref{app:exp} provides a simple sufficient condition to verify \prettyref{ass:reg_strong}.
\end{enumerate}

\subsection{Weak Recovery}

The following theorem is our main result about weak recovery.  It gives a sufficient condition
and a matching necessary condition for weak recovery.

\begin{theorem}\label{thm:weak_general}
Suppose \prettyref{ass:reg_strong} holds. If
	\begin{equation}
	K  \cdot D(P\|Q) \to \infty   \quad \text{and} \quad \liminf_{n\to\infty} \frac{(K-1)  D(P\|Q)}{\log \frac{n }{K}}> 2,
	\label{eq:weak-bdd_suff}
\end{equation}
then
$$
\Prob\{|\widehat{C}_{\rm ML} \triangle C^*| \le 2 K \epsilon \} \geq 1 - \eexp^{ - \Omega( K/ \epsilon  )},
$$
where $\epsilon = 1/\sqrt{ K D( P \| Q)}$.

If there exists $\hat \xi$ such that  $\Expect[d_H(\xi, \hat \xi)] = o(K)$, then
\begin{equation}
	K \cdot D(P\|Q) \to \infty   \quad \text{and} \quad \liminf_{n\to\infty} \frac{(K-1) D(P\|Q)}{\log \frac{n }{K}} \ge 2.
	\label{eq:weak-bdd_nec}
	\end{equation}
\end{theorem}

\begin{remark}  \label{rmk:KvsK-1_weak}
The assumption $K\geq 2,$ implies $K/2 \leq K-1 \leq K$, so the first parts of \eqref{eq:weak-bdd_suff}
and \eqref{eq:weak-bdd_nec} would have the same meaning if $K$ were replaced by $K-1$.    
In the special case of bounded LLR, the factor $K-1$  in the second parts of \eqref{eq:weak-bdd_suff}
and  \eqref{eq:weak-bdd_nec} can be replaced by $K$.  This is because if  $\log \frac{d P}{d Q}$ is bounded,  so is $D(P\|Q) $, and
$K D(P\|Q) \to \infty$ implies  $K\to \infty$ and hence also $(K-1)/K \to 1$.
 \end{remark}

\begin{corollary}[Weak recovery in Bernoulli case] \label{cor:weak_Bern}
Suppose the ratios $\log \frac{p}{q}$ and $\log \frac{\bar p}{\bar q}$ are
bounded.   If
\begin{align}
K \cdot d(p \| q) \to \infty \quad \text{and} \quad  \liminf_{n \to \infty} \frac{K d(p\| q) }{\log \frac{n}{K} } > 2, \label{eq:MLE_comm_suff_cond2}
\end{align}
then weak recovery is possible. If weak recovery is possible, then
\begin{align}
K \cdot d(p \| q) \to \infty \quad \text{and} \quad  \liminf_{n \to \infty} \frac{ K d(p\| q) }{\log \frac{n}{K} } \ge 2. \label{eq:MLE_comm_nec_cond2}
\end{align}
\end{corollary}
\begin{remark} \label{rmk:KvsKminus_one}
Condition \prettyref{eq:MLE_comm_nec_cond2} is necessary even if  $p/q \to \infty,$ but \prettyref{eq:MLE_comm_suff_cond2} alone is not sufficient without
the assumption that $p/q$ is bounded.
This can be seen by considering the extreme case where $K=n/2$, $p=1/n$, and $q=\eexp^{-n}$. In this case, condition  \prettyref{eq:MLE_comm_suff_cond2} is clearly satisfied; however, the subgraph induced by index in the cluster is an \ER random graph
with edge probability $1/n$ which contains at least a constant fraction of isolated vertices
with probability converging to one as $n\to \infty$. It is not possible to correctly determine
whether the isolated vertices are in the cluster, hence the impossibility of weak recovery.
\end{remark}

\begin{corollary}[Weak recovery in Gaussian case] \label{cor:weak_Gau}
If
 \begin{equation}
K \mu^2  \rightarrow \infty \quad \text{and} \quad \liminf _{n\to\infty} \frac{(K-1) \mu^2}{\log \frac{n }{K}}> 4,
	\label{eq:weak_Gaussian_suff}
\end{equation}
then weak recovery is possible. If weak recovery is possible, then
 \begin{equation}
K \mu^2  \rightarrow \infty \quad \text{and} \quad \liminf _{n\to\infty} \frac{(K-1) \mu^2}{\log \frac{n }{K}} \ge 4.\label{eq:weak_Gaussian_nec}
\end{equation}
\end{corollary}


\subsection{Exact Recovery}
The following theorem states our main result about exact recovery.  It gives a sufficient condition
and a matching necessary condition for exact recovery.   Since exact recovery implies weak recovery, conditions
from \prettyref{thm:weak_general} naturally enter.

\begin{theorem} \label{thm:exact_general}
Suppose \prettyref{ass:reg_strong} holds.
If  \eqref{eq:weak-bdd_suff}  and the following hold:
\begin{equation}
\liminf_{n\to\infty} \frac{K E_Q\pth{\frac{1}{K} \log \frac{n}{K}}}{\log n} > 1.
	\label{eq:voting-suff}
\end{equation}
then the maximum likelihood estimator satisfies $\pprob{ \CML = C^*} \to 1$.

If there exists an estimator $\hat C$ such that $\pprob{ \hat C = C^*} \to 1,$ then
\eqref{eq:weak-bdd_nec} and the following hold:
\begin{equation}
\liminf_{n\to\infty} \frac{K E_Q\pth{\frac{1}{K} \log \frac{n}{K}}}{\log n} \ge 1.
	\label{eq:voting-nec}
\end{equation}
\end{theorem}

\begin{remark}
\label{rmk:CPQ}	
	In the special case of linear community size, \ie, $K=\Theta(n)$, \prettyref{eq:voting-suff} and \prettyref{eq:voting-nec} can be simplified by replacing 
$E_Q\pth{\frac{1}{K} \log \frac{n}{K}}$ by the Chernoff index between $P$ and $Q$ \cite{Chernoff52}:
\begin{equation}
E_P(0)=E_Q(0)=\sup_{0\leq \lambda \leq 1} -\log \int  \left(\frac{dP}{dQ}\right)^\lambda dQ \triangleq C(P,Q).
	\label{eq:CPQ}
\end{equation}
To see this, note that in the definition $E_Q(\theta)$ in \prettyref{eq:ratefunction} the supremum can be restricted to $\lambda\in[0,1]$ and hence 
$E_Q(\theta)\leq E_Q(\theta+\delta)\leq E_Q(\theta)+\delta$ as long as $-D(Q\|P)\leq \theta\leq\theta+\delta\leq D(P\|Q)$. By \eqref{eq:weak-bdd_suff}, $\delta =\frac{1}{K}\log\frac{n}{K} \leq D(P\|Q)$ for all sufficiently large $n$. Hence, in the case of $K=\Theta(n)$, $ C(P, Q )\le E_Q\pth{\frac{1}{K} \log \frac{n}{K}}\le C(P, Q) + \Theta(\frac{1}{n}),$ proving the claim.
The Chernoff index $C(P,Q)$ gives the optimal exponent for decay of sum of error probabilities for the binary hypothesis testing problem in the large-sample limit.
\end{remark}

\begin{corollary}[Exact recovery in Bernoulli case]\label{cor:planted_dense_exact}
Suppose $\log \frac{p}{q}$ and $\log \frac{\bar p}{\bar q}$ are bounded.
If \prettyref{eq:MLE_comm_suff_cond2} holds,
and
\begin{align}
\liminf_{n \to \infty}  \frac{ K   d (\tau^\ast  \| q) }{\log n } >1,    \label{eq:planted_dense_exact_suff1}
\end{align}
where
\begin{align}
\tau^\ast = \frac{ \log \frac{\bar q}{\bar p} + \frac{1}{K} \log \frac{n}{K} }{\log \frac{p\bar q}{q\bar p } }, \label{eq:deftau}
\end{align}
then exact recovery is possible. If exact recovery is possible, then \prettyref{eq:MLE_comm_nec_cond2} holds,
and
\begin{align}
\liminf_{n \to \infty}  \frac{ K   d (\tau^\ast  \| q) }{\log n } \geq 1.  \label{eq:planted_dense_exact_necc1}
\end{align}
\end{corollary}
\begin{proof}
In the Bernoulli case, 
$E_P(\theta)=d(\alpha\|p)$ and $E_Q(\theta)=d(\alpha\|q)$, where $\alpha = (\theta + \log \frac{\bar q}{\bar p})/\log \frac{p\bar q}{q\bar p}.$
\end{proof}

\begin{remark} \label{rmk:bernoulli_exact_threshold}
Consider the Bernoulli case in the regime 
\[
K=  \frac{\rho n }{\log^{s-1} n}, \quad p= \frac{a\log^s n}{n}, \quad  q=  \frac{b\log^s n}{n},
\]
where
$s\ge 1$ is fixed, $\rho \in(0,1) $ and $a>b>0$. Let $I(x, y) \triangleq x - y \log (\eexp x/y) $ for $x, y>0$. 
Then the sharp recovery thresholds are determined by Corollaries \ref{cor:weak_Bern} and \ref{cor:planted_dense_exact} 
as follows: For any $\epsilon > 0$,
\begin{itemize}
\item For $s>1$, if $\rho I ( b, a) \ge {  \frac{(2+\epsilon) (s-1) \log \log n }{\log n} }$, then weak recovery is possible;
if $\rho I(b,a) \le {  \frac{(2-\epsilon) (s-1) \log \log n }{\log n} }$, then weak recovery is impossible. 
For $s=1$, weak recovery is possible if and only if $\rho I ( b, a) = \omega(\frac{1}{\log n})$.

\item Assume $\rho, a, b$ are fixed constants. Let $\tau_0=(a-b)/\log (a/b)$. 
Then exact recovery is possible if $\rho I(b,\tau_0)>1$; conversely, 
if $\rho I(b,\tau_0)<1$, then exact recovery is impossible, generalizing the previous results of 
\cite{HajekWuXuSDP14,AbbeSandon15} for linear community size ($s=1$). 
 To see this, note that 
by definition,  $\tau^\ast= (1+ o(1)) \tau_0 \log^s n /n $, and thus
$d(\tau^\ast \| q) = (1+ o(1) ) I(b,\tau_0) \log^s n /n$. 
\end{itemize}

\end{remark}

\begin{remark}
The recent work \cite{JL15} considered a generalized planted bisection model where $A_{ij}\sim P$ if $i,j$ are in the same community and $Q$ if otherwise. Their result applies to the following generalization of the Bernoulli distribution, where $P=(p_0,\ldots,p_m)$ and $Q=(q_0,\ldots,q_m)$ with $p_i=\frac{a_i \log n}{n},q_i=\frac{b_i \log n}{n}, 1\leq i \leq m$ for some $m \geq 1$ and positive constants $a_i, b_i,$ $1\leq i \leq m$. For this family of distribution the LLR is bounded and hence \prettyref{thm:exact_general} gives the sharp condition for recovering a single hidden community. Specifically, note that $\psi_Q(\lambda) = \big(\sum_{i=1}^m  a_i^{\lambda} b_i^{\bar{\lambda}} - a_i\lambda-b_i\bar{\lambda} +o(1) \big)  \frac{\log n}{n}$. Thus for $K=\rho n$ with a fixed $\rho$, the sharp threshold of exact recovery is given by  $\rho \sup_{0 < \lambda < 1} \big(\sum_{i=1}^m  a_i\lambda+b_i\bar{\lambda}-a_i^{\lambda} b_i^{\bar{\lambda}} \big) > 1$. For $m=1$ with $a_1=a$ and $b_1=b$, the optimal $\lambda$ is determined by $a^\lambda b^{\bar{\lambda}} = (a-b)/\log (a/b) =\tau_0$, and the  sharp threshold of exact recovery simplifies to $\rho I(b, \tau_0)>1$, recovering the
result for the Bernoulli case given in \prettyref{rmk:bernoulli_exact_threshold}.	
\end{remark}

\begin{corollary}[Exact recovery in Gaussian case]  \label{cor:nec_exact_submat}
If \prettyref{eq:weak_Gaussian_suff} holds and
\begin{equation}
\liminf_{n \to \infty} \frac{ K \mu^2  }{  \left( \sqrt{ 2 \log n} + \sqrt{ 2  \log K } \right)^2 }  > 1,
	\label{eq:submat-mle-suff}
\end{equation}
then exact recovery is possible. If exact recovery is possible, then \prettyref{eq:weak_Gaussian_nec} holds
and
\begin{equation}
\liminf_{n \to \infty} \frac{ K \mu^2  }{ \left( \sqrt{ 2 \log n} + \sqrt{ 2  \log K } \right)^2 }  \ge 1.
	\label{eq:submat-mle-nece}
\end{equation}
\end{corollary}
See \prettyref{app:Gauss_exact_cor} for a proof of \prettyref{cor:nec_exact_submat}.

\begin{remark}
Consider the  Gaussian case in the regime 
$$
K=  \frac{\rho n }{\log^{s-1} n}, \quad  \mu^2= \frac{\mu_0^2 \log^s n}{n},$$  
where $s\ge 1$ and $\rho \in (0,1)$ are fixed constants.
The critical signal strength that allows weak or exact recovery is determined by Corollaries  \ref{cor:weak_Gau} and \ref{cor:nec_exact_submat} as follows: For any $\epsilon > 0$,
\begin{itemize}
\item For $s>1$, if $ \mu_0  > (2+\epsilon) \sqrt{ \frac{(s-1) \log \log n} { \rho \log n} }$, then weak recovery is possible;
conversely, if $ \mu_0 < (2-\epsilon) \sqrt{ \frac{(s-1) \log \log n} { \rho \log n} }$, then weak recovery is impossible.
For $s = 1$, weak recovery is possible if and only if $\mu_0 = \omega(\frac{1}{\sqrt{\log n}})$.
\item If $\mu_0 >  \sqrt{ \frac{8+\epsilon}{\rho}}$, then exact recovery is possible; conversely, 
If $\mu_0 <  \sqrt{ \frac{8-\epsilon}{\rho}}$, then exact recovery is impossible.
\end{itemize}
\end{remark}

\begin{remark}   \label{rmk:Butucea_sharp}
Butucea et al. \cite{Butucea2013sharp} considers the submatrix
localization model with an $n\times m$ submatrix with an
elevated mean in an $N\times M$ large Gaussian random matrix with
independent entries, and  gives sufficient conditions and necessary conditions,
matching up to constant factors, for exact recovery, which are
analogous to those of \prettyref{cor:nec_exact_submat}.
Setting $(n, m, N, M)$ in  \cite[(2.3)]{Butucea2013sharp}
(sufficient condition for exact recovery of rectangular submatrix)
equal to $(K, K, n, n)$ gives precisely the sufficient condition of
\prettyref{cor:nec_exact_submat} for exact recovery of a principal
submatrix of size $K$ from symmetric noise.   This coincidence
can be understood as follows.  The nonsymmetric observations
of  \cite[(2.3)]{Butucea2013sharp} in the case of parameters $(K, K, n, n)$
yield twice the available information as the symmetric observation matrix we
consider (diagonal observations excluded) while the amount
of information required to specify a $K\times K$ (not necessarily
principal) submatrix of an $n\times n$ matrix is twice
the information needed to specify a principal one.
 The proof techniques
of \cite{Butucea2013sharp} are similar to ours, with the main
difference being that we simultaneously investigate conditions
for weak and exact recovery.
Finally, the information limits of weak recovery for biclustering are established in \cite[Section 4.1]{HajekWuXu_MP_submat15}
based on modifications of the arguments in \cite{Butucea2013sharp}.
\end{remark}

\begin{remark} 
If $K \leq n^{1/9}$, \prettyref{eq:weak_Gaussian_suff}
implies \prettyref{eq:submat-mle-suff}, and thus  \prettyref{eq:weak_Gaussian_suff} alone is sufficient for exact recovery;
if $K\geq  n^{1/9}$, then \prettyref{eq:submat-mle-suff}
implies \prettyref{eq:weak_Gaussian_suff}, and  \prettyref{eq:submat-mle-suff} alone
is sufficient for exact recovery. 
\end{remark}

The reminder of the paper is organized as follows.   \prettyref{sec:technical} gives some preliminaries.
\prettyref{sec:weakGeneral} proves \prettyref{thm:weak_general}, pertaining to weak recovery,
and  \prettyref{sec:exactgeneral} proves \prettyref{thm:exact_general}, pertaining to exact recovery.
Additional results are introduced in  \prettyref{sec:exactgeneral}, which highlight alternative sufficient
and necessary conditions for exact recovery involving large deviation probabilities for  sums
of random variables, related to the voting procedure mentioned in the introduction.

\section{On the Assumptions on $P$ and $Q$}  \label{sec:technical}

This section presents some conditions sufficient for \prettyref{ass:reg_strong},
and some implications of \prettyref{ass:reg_strong}.

\begin{lemma}[Bounded LLR]    \label{lmm:scale}
If $|L| \leq B$ for some positive constant $B$, then \prettyref{ass:reg_strong} holds with
 $C= 2\eexp^{5B}$. 
\end{lemma}
\begin{proof}
First, some background.
Let $\phi(y)=e^y - 1 -y$, which is nonnegative, convex, with $\phi(0)=\phi'(0)=0$
and $\phi''(y)=e^y$.   Thus for $|y|\leq B$, $e^{-B} \leq \phi''(y) \leq e^B$ and hence
$\frac{e^{-B} y^2}{2} \leq \phi(y) \leq \frac{e^B y^2}{2}$.

Now to the proof.  We begin by noticing that for all $\lambda \in [-1,1]$,
\begin{align*}
\psi''_Q(\lambda)  & = \var_{Q_{\lambda}}(L) \leq \Expect_{Q_{\lambda}}[L^2]  
= \frac{\Expect_Q[L^2\eexp^{\lambda L}]}{\Expect_Q[\eexp^{\lambda L}]} \leq \eexp^{2B} \Expect_Q[L^2].
\end{align*}
In turn, using $y^2 \leq 2e^B \phi(y)$  as shown above and recalling that $L=\log \frac{dP}{dQ}$, we have
\begin{align*}
\Expect_Q[L^2] \leq 
2e^B   \Expect_Q[\phi(L)] =  2e^B  D(Q\| P).
\end{align*}
Combining the last two displayed equations yields $\psi''_Q(\lambda)   \leq 2\eexp^{3B}D(Q\|P)$ for
$\lambda \in [-1,1]$.    
 Abbreviate $\psi_Q$ by $\psi$. By a variation of the argument above, we have
\begin{align*}
\psi''(\lambda)  & = \var_{Q_{\lambda}}(L) \leq \Expect_{Q_{\lambda}}[L^2]  
= \frac{\Expect_Q[L^2\eexp^{\lambda L}]}{\Expect_Q[\eexp^{\lambda L}]} \leq \eexp^{4B} \Expect_Q[L^2] ~~~ \mbox{if } \lambda \in [0,2],
\end{align*}
so that $\psi''(\lambda) \leq 2\eexp^{5B} D(Q\|P)$ for $\lambda \in [0, 2]$.   
Let $\tilde \psi$ denote the version of $\psi$ that would be obtained if the roles of $P$ and $Q$ were swapped.   
Then
$\tilde \psi''(\lambda) \leq 2\eexp^{5B} D(P\|Q)$ for $\lambda \in [0, 2]$. Since $\psi$ and $\tilde \psi$ are related by
reflection about $\lambda=1/2$:  $\psi(\lambda)\equiv \tilde \psi(1-\lambda)$, we have
$\psi''(\lambda) \leq 2\eexp^{5B} D(P\|Q)$ for $\lambda \in [-1,1]$, completing the proof.
\end{proof}

As shown in the proofs,
\prettyref{thm:weak_general} (weak recovery), and the sufficiency part of \prettyref{thm:exact_general} (exact recovery)
hold under assumptions somewhat weaker than  \prettyref{ass:reg_strong}; only the necessity part of \prettyref{thm:exact_general} relies on \prettyref{ass:reg_strong}.  To clarify this subtlety, we introduce two successively weaker assumptions. We also provide a lemma
showing that any of the assumptions imply  the equivalence $D(P\|Q) \asymp D(Q\| P) \asymp C(P,Q) $.

\begin{assumption}   \label{ass:weaker_regularity}
For some constant $C$:
\begin{align}
\psi_P(\lambda) - D(P\|Q) \lambda  \leq & ~ \frac{C D(P\|Q)}{2}  \lambda^2, \quad \lambda \in [-1,0] \label{eq:subgP-loc}	\\
\psi_Q(\lambda) + D(Q\|P) \lambda  \leq & ~ \frac{C D(Q\|P)}{2}  \lambda^2, \quad \lambda \in [-1,1] \label{eq:subgQ-loc}	
\end{align}
\end{assumption}

\begin{remark}
\prettyref{ass:weaker_regularity} is weaker than the assumption that $L$ is sub-Gaussian with scale parameter
$D(P\|Q)$  under $P$ and with scale parameter $D(Q\|P)$ under $Q.$  A sub-Gaussian assumption
would correspond to requiring \prettyref{eq:subgP-loc} and \prettyref{eq:subgQ-loc} to hold for all $\lambda \in \reals$.
\end{remark}

\begin{assumption}  \label{ass:weakest_regularity}
For some constant $C$:
\begin{align}
E_P((1-\eta) D(P\|Q)) &\geq \frac{\eta^2}{2C} D(P\|Q),  \quad \eta \in [0,1] \label{eq:divquadassump} \\
E_Q( - (1-\eta) D(Q\|P)) &\geq \frac{\eta^2}{2C} D(Q\|P), \quad \eta \in [0,1].   \label{eq:divquadassump2}
\end{align}
\end{assumption}


\begin{lemma}  \label{lmm:equivalence}
\prettyref{ass:reg_strong} implies \prettyref{ass:weaker_regularity} which
implies \prettyref{ass:weakest_regularity}, with the same constant
$C$ throughout.   Any of these assumptions implies that:
\begin{align}   \label{eq:equivalence}
\min\{   D(P\|Q), D(Q\| P)  \} \geq  C(P,Q) \geq \frac{1}{2C} \max\{   D(P\|Q), D(Q\| P)  \},
\end{align}
and hence also that $D(P\|Q) \asymp D(Q\| P)  \asymp C(P,Q) $.
\end{lemma}
\begin{proof}
\prettyref{ass:reg_strong} $\implies$ \prettyref{ass:weaker_regularity}:
Condition \eqref{eq:subgP-loc}  is implied by \prettyref{ass:reg_strong} because $\psi_P(0)=0,$ and $\psi_P'(0)= D(P\|Q)$,
so by the integral form of Taylor's theorem,  $\psi_P(\lambda) - D(P\|Q) \lambda $ is $\lambda^2/2$ times a weighted average of
$\psi_P''$ over the interval $[\lambda, 0]$ for $\lambda \in [-1,0]$.  Similarly, \eqref{eq:subgQ-loc} is implied by
\prettyref{ass:reg_strong} because  $\psi_Q(\lambda) + D(Q\|P)\lambda$ is a weighted average
 of $\psi_Q''$ over the interval with endpoints $0$ and $\lambda,$  for $\lambda \in [-1,1]$.
 
\prettyref{ass:weaker_regularity} $\implies$  \prettyref{ass:weakest_regularity}:
Since $\psi_P(-1)=\psi_Q(1)=0,$  either \eqref{eq:subgP-loc} or \eqref{eq:subgQ-loc} imply that $C\geq 2,$ which
is achieved in the Gaussian case.   Condition \prettyref{eq:subgP-loc} implies
\begin{align*}
E_P((1-\eta) D(P\|Q)) &= \sup_{\lambda\in\reals}(\lambda (1-\eta) D(P\|Q) - \psi_P(\lambda))  \\
& \geq D(P\|Q) \sup_{\lambda\in\reals} \pth{- \lambda \eta  - \frac{C\lambda^2}{2}} = \frac{\eta^2}{2C} D(P\|Q),
\end{align*}
where the supremum is attained at $\lambda = \frac{-\eta}{C}$ which belongs to $[-1,0]$ by the
fact $C\geq 2$.   So \eqref{eq:subgP-loc} implies \eqref{eq:divquadassump}.
The proof that \eqref{eq:subgQ-loc} implies  \eqref{eq:divquadassump2} is similar.

\prettyref{ass:weakest_regularity} $\implies$ \eqref{eq:equivalence}:   Taking $\eta=1$ in \eqref{eq:divquadassump}
and \eqref{eq:divquadassump2} we get $C(P,Q) \geq  \frac{1}{2C} \max\{   D(P\|Q), D(Q\| P)  \}$.	
In the other direction, $D(P\|Q) = E_Q(D(P\|Q)) \geq E_Q(0)=C(P,Q)$ and, similarly,
$ D(Q\|P)  \geq C(P,Q)$. 
\end{proof}

Recall the Chernoff upper bounds \prettyref{eq:LDupper1} and \prettyref{eq:LDupper2} on the probability of large deviations, which hold non-asymptotically for any sample size $n$ and any pair $P$ and $Q$. To prove the necessary condition for exact recovery, we need a lower bound with matching exponent. Such a result is well-known for fixed distributions. Indeed, the sharp asymptotics of large deviation is given by the Bahadur-Rao theorem (see, \eg, \cite[Theorem 3.7.4]{DZ98}); however, this result is not applicable in the hidden community problem because  both $P$ and $Q$ can vary with $n$. The following lemma provides a non-asymptotic information-theoretic lower bound (cf.~\cite[Theorem 11.1]{PW-it} and \cite[Eq.~(5.21), p.~167]{ckbook}):
\begin{lemma}
\label{lmm:ld-lb}	
If $-D(Q\|P) \leq \gamma < \gamma + \delta \leq D(P\|Q)$, then 
	\begin{align}
	\exp\left( -n E_Q(\gamma) \right) \geq 
Q\qth{\sum_{k=1}^n L_k > n \gamma} \geq \exp\left( - \frac{n E_Q(\gamma + \delta) + \log2}{1- \frac{1}{n\delta^2 } \sup_{0 \leq \lambda \leq 1} \psi_Q''(\lambda) } \right).
\label{eq:ld-lb}
\end{align}
\end{lemma}
\begin{proof}
The left inequality in \eqref{eq:ld-lb} is the Chernoff bound \prettyref{eq:LDupper1}; it remains to prove the right inequality.
Let $E_n = \sth{\sum_{k=1}^n L_k > n \gamma}$.
For any $Q'$, the data processing inequality of KL divergence gives
\begin{align*}
d(Q'[E_n]\|Q[E_n]) \leq D(Q'^n \|Q^n) = nD(Q'\|Q).
\end{align*}
Using the lower bound for the binary divergence $d(p\|q) =-h(p) + p \log \frac{1}{q} + (1-p) \log \frac{1}{1-q} \geq - \log 2 + p \log \frac{1}{q}$ yields
\begin{align*}
d(Q'[E_n]\|Q[E_n]) \geq -\log 2 + Q'[E_n]\log\frac{1}{Q[E_n]},
\end{align*}
so that
\begin{align*}
Q[E_n] \geq \exp\left( \frac{-nD(Q'\|Q) - \log2}{Q'[E_n]} \right).
\end{align*}
For $\lambda \in [0,1]$, the tilted distribution $Q_{\lambda}$ is given by $dQ_\lambda = \frac{\exp(\lambda L) dQ}{\Expect_Q[\exp(\lambda L)]}=
\frac{P^\lambda Q^{1-\lambda}}{\int P^\lambda Q^{1-\lambda}}$. 
Then for any $\alpha \in [-D(Q\|P), D(P\|Q)]$, there exits a unique $\lambda  \in [0,1]$, such that 
$\Expect_{Q_\lambda}[L] = \alpha$ and $E_Q(\alpha) =  \psi_Q^*(\alpha)  = D(Q_\lambda \| Q)$.
Choosing $\alpha=\gamma +\delta$ and $Q'=Q_{\lambda}$, we have 
\begin{align*}
1 - Q_\lambda [E_n] & = Q_\lambda  \qth{\sum_{k=1}^n L_k \leq n \gamma} = Q_\lambda\qth{\sum_{k=1}^n (L_k - \Expect_{Q'}[L_k]) \leq - n \delta} \\
& \leq \frac{\var_{Q_\lambda} ( L_1) }{n \delta^2} =  \frac{ \psi_Q''(\lambda)  }{n \delta^2}.
\end{align*}
Consequently,
$$
Q\qth{\sum_{k=1}^n L_k > n \gamma} \geq \exp\left( - \frac{n E_Q(\gamma + \delta) + \log2}{1- \frac{ \psi_Q''(\lambda)}{n\delta^2 }} \right).
$$
\end{proof}

\begin{corollary}  \label{cor:LD_lower_bnd}
If \prettyref{ass:reg_strong} holds and  $-D(Q\|P) \leq \gamma < \gamma + \delta \leq D(P\|Q)$:
\begin{align*}
	\exp\left( -n E_Q(\gamma) \right) \geq 
Q\qth{\sum_{k=1}^n L_k > n \gamma} \geq \exp\left( - \frac{n E_Q(\gamma + \delta) + \log2}
{1- \frac{  C\min\{D(P\|Q), D(Q\|P)\}      }{n\delta^2 }} \right).
\end{align*}
\end{corollary}

\section{Weak Recovery for General $P/Q$ Model}\label{sec:weakGeneral} \label{sec:weak}

\prettyref{thm:weak_general} is proved in \prettyref{sec:necsuff_weak}.
\prettyref{sec:suff_weak_random} provides a modification of  the sufficiency part
of \prettyref{thm:weak_general} giving a sufficient condition for weak recovery
with random cluster size; it is used in \prettyref{sec:exactgeneral} to prove sufficient
conditions for exact recovery. 

\subsection{Proof of \prettyref{thm:weak_general}} \label{sec:necsuff_weak}

\begin{remark}
The sufficiency proof only uses \prettyref{eq:divquadassump} while the
necessity proof only uses \prettyref{eq:divquadassump2}.
The sufficiency proof is based on analyzing the MLE via a delicate application of union bound and
large deviation upper bounds \prettyref{eq:LDupper1} and \prettyref{eq:LDupper2}. For the necessary part, the proof for 
the first condition in \prettyref{eq:weak-bdd_nec} uses a genie argument and the theory of binary hypothesis testing,
while the proof of the second condition in \prettyref{eq:weak-bdd_nec} is based on mutual information and rate-distortion function.
\end{remark}

\paragraph{Sufficiency}
We let $\widehat{C}$ denote the MLE,  $\widehat{C}_{\rm ML},$
for brevity in the proof.
Let $L=|\widehat{C} \cap C^*|$ and $\epsilon = 1/\sqrt{ K D( P \| Q)}$.
Since $K\geq 2$ and $(K-1)  D(P\|Q) \to \infty$ by assumption,
we have $\epsilon= o(1).$ 
Since $|\widehat{C}|=|C^*|=K$ and hence
$|\widehat{C} \triangle C^*| = 2 (K-L)$, it suffices to show that $\Prob\{L \leq (1-\epsilon) K\} \le \exp( - \Omega( K/ \epsilon  ) )$.

Note that
\begin{align}
&~ e(\widehat{C},\widehat{C})-e(C^*,C^*)  =  e(\widehat{C}\backslash C^*, \widehat{C}\backslash C^*) +  e(\widehat{C}\backslash C^*, \widehat{C}\cap C^*) -  e(C^* \backslash \widehat{C}, C^*).  \label{eq:eC1}
\end{align}
and $|C^*\backslash \widehat{C}|=|\widehat{C}\backslash C^*|=K-L$.
Fix $\theta \in [-D(Q\|P), D(P\|Q)]$ whose value will be chosen later.
 Then for any $0 \le \ell \le K-1$,
\begin{align*}
  \{ L = \ell \}  \subset &~  \{  \exists C \subset [n]: |C| = K, |C\cap C^*| = \ell, e(C, C) \geq e(C^*,C^*) \} \nonumber \\
= &~   \{  \exists S \subset C^*, T \subset (C^*)^c: |S| = |T| =K-\ell,    e(S, C^*)                          \le e(T, T) +  e(T, C^*\backslash S)   \} \nonumber \\
\subset &~   \{  \exists S \subset C^*: |S| =K-\ell,   e(S, C^*)             \leq m \theta  \} \nonumber \\
&~\cup \{ \exists S \subset C^*, T \subset (C^*)^c: |S|=|T| =K-\ell, e(T,  T) +  e(T, C^\ast \backslash S) \ge m \theta \},
\end{align*}
where $m=  \binom{K}{2}-\binom{\ell}{2}$.
Notice that $e(S,C^*)$ has the same distribution as $ \sum_{i=1}^m L_i$ under measure $P$;
$e(T,  T) + e(T, C^\ast \backslash S)$ has the same distribution as $ \sum_{i=1}^m L_i$ under measure $Q$
where $L_i$ are i.i.d.\ copies of $\log \frac{ \diff P}{ \diff Q} $.
Hence, by  the union bound and the large deviation bounds \prettyref{eq:LDupper1} and \prettyref{eq:LDupper2},
\begin{align*}
\prob{L=\ell} & \le  \binom{K}{K-\ell} P \qth{ \sum_{i=1}^m L_i \leq m \theta } + \binom{n-K}{K-\ell} \binom{K}{K-\ell}
Q \qth{ \sum_{i=1}^m L_i \geq m \theta } \\
& \le  \binom{K}{K-\ell} \exp(- m E_P(\theta)) +
\binom{n-K}{K-\ell} \binom{K}{K-\ell} \exp(- m E_Q(\theta)) \\
& \le  \left( \frac{K\eexp }{K-\ell}  \right) ^{K-\ell}  \exp(- m E_P(\theta)) +
\left(  \frac{(n-K) K \eexp^2 }{( K-\ell )^2} \right) ^ {K -\ell}  \exp(- m E_Q(\theta))
\end{align*}
where the last inequality holds due to the fact that $\binom{a}{b} \le (\eexp a/b)^b$.
Notice that $m = (K-\ell) (K+\ell-1)/2 \ge (K-\ell)(K-1)/2$.
Thus, for any $\ell \leq (1-\epsilon)K$,
\begin{align}
\prob{L=\ell}  \le \eexp^{-(K-\ell) E_1} + \eexp^{-(K-\ell) E_2}, \label{eq:weakProbupperbound}
\end{align}
where
\begin{align*}
E_1 &\triangleq (K-1) E_P(\theta) /2 - \log \frac{\eexp}{\epsilon} , \\
E_2 &\triangleq (K-1) E_Q(\theta) /2  - \log \frac{(n-K) \eexp^2}{K \epsilon^2}.
\end{align*}
By the assumption \prettyref{eq:weak-bdd_suff}, we have $(K-1)  D(P\|Q) (1-\eta) \geq 2  \log \frac{n}{K}$ for some $\eta \in (0,1)$. Choose $\theta = (1-\eta) D(P\|Q)$.
By the assumption \prettyref{eq:divquadassump}, we have
\begin{align*}
E_1 \geq c  \eta^2 (K-1)  D(P\|Q) /2  - \log \frac{\eexp}{\epsilon}.
\end{align*}
Using the fact that $E_P(\theta)=E_Q(\theta)-\theta$, we have
\begin{align*}
E_2 & \geq c \eta^2 (K-1)  D(P\|Q) /2  - 2 \log \frac{\eexp}{\epsilon} + \frac{(K-1) }{2} D(P\|Q) (1-\eta) -  \log \frac{n-K}{K} \\
 & \geq c \eta^2 (K-1)  D(P\|Q) /2   - 2 \log \frac{\eexp}{\epsilon}.
\end{align*}
Therefore, in view of $\epsilon = 1/\sqrt{ K D( P \| Q)}$, it follows that $E \triangleq \min\{E_1,E_2\} = \Omega(KD(P\|Q)) =\Omega(\epsilon^{-2})$.
Hence, in view of \prettyref{eq:weakProbupperbound},
\begin{align*}
\prob{L \le (1-\epsilon) K}  & = \sum_{\ell=0}^{(1-\epsilon) K}\prob{L = \ell}
 \leq \sum_{\ell=\epsilon K}^{\infty} \left( \eexp^{-\ell E_1} + \eexp^{-\ell E_2} \right)  \\
&  \leq \frac{2 \exp(-\epsilon K E)}{1 - \exp(-E)} = \exp( - \Omega(K/\epsilon )). \qedhere
\end{align*}

\paragraph{Necessity}
Given $i,j \in [n]$, let $\xi_{\backslash i,j}$ denote $\{\xi_k\colon k \neq i,j\}$.
Consider the following binary hypothesis testing problem for determining $\xi_i$.
If $\xi_i=0$, a node $J$ is randomly and uniformly chosen from $\{j\colon \xi_j=1\}$, and we observe $(A, J, \xi_{\backslash i, J })$;
if $\xi_i=1$, a node $J$ is randomly and uniformly chosen from $\{j\colon \xi_j=0 \}$, and we observe $(A, J, \xi_{\backslash i, J } )$.
Note that
\begin{align*}
\frac{ \prob{ J, \xi_{\backslash i, J}, A | \xi_i=0} }{ \prob{ J, \xi_{\backslash i, J}, A | \xi_i=1}} =  \frac{ \prob{ \xi_{\backslash i, J}, A | \xi_i=0, J} }{ \prob{ \xi_{\backslash i, J}, A | \xi_i=1, J}}
=\frac{ \prob{ A | \xi_i=0, J, \xi_{\backslash i, J} } }{ \prob{A | \xi_i=1, J,  \xi_{\backslash i, J} }} = \prod_{k \in [n]\backslash\{i, J\}: \xi_k=1 }
\frac{Q (A_{ik} ) P(A_{Jk}) }{P(A_{ik}) Q(A_{Jk}) },
\end{align*}
where the first equality holds because $\prob{ J | \xi_i=0} = \prob{ J| \xi_i =1}$;
the second equality holds because $\prob{ \xi_{\backslash i, J} | \xi_i=0, J} =\prob{ \xi_{\backslash i, J} | \xi_i=1, J }$.
Let $T$ denote the vector consisting of $A_{ik}$ and $A_{Jk}$ for all $ k \in [n]\backslash\{i, J\}$ such that $\xi_k=1$.
Then $T$ is a sufficient statistic of $(A, J, \xi_{\backslash i, J} )$
for testing $\xi_i=1$ and $\xi_i=0$. Note that if $\xi_i=0$,
$T$ is distributed as $Q^{\otimes (K-1)} P^{\otimes (K-1) } $;
if $\xi_i=1$, $T$ is distributed as $P^{\otimes (K-1)} Q^{\otimes (K-1)} $. Thus, equivalently, we are testing $Q^{\otimes (K-1)} P^{\otimes (K-1)}$ versus
 $P^{\otimes (K-1)} Q^{\otimes (K-1)}$; let $\calE$ denote the optimal average probability of testing error.
Then we have the following chain of inequalities:
\begin{align}
\Expect[d_H(\xi, \hat \xi)]
\geq & ~ \sum_{i=1}^n \min_{\hat \xi_i(A)} \Prob[\xi_i \neq  \hat \xi_i] 	
\geq  \sum_{i=1}^n \min_{\hat \xi_i ( A, J, \; \xi_{\backslash i, J})} \Prob[\xi_i \neq  \hat \xi_i] 	\nonumber \\
= & ~ n \min_{\hat \xi_1(A, J, \; \xi_{\backslash 1, J})} \Prob[\xi_1 \neq  \hat \xi_1] 	
= n \cal E.
\end{align}
By the assumption $\Expect[d_H(\xi, \hat \xi)]  =o(K)$, it follows that
$\calE =o(K/n)$.
Since $K/n$ is bounded away from one, this implies that the sum of Type-I and II probabilities of error
$p_{e,0} + p_{e,1} = o(1)$, which is equivalent to
$\text{TV}((P\otimes Q)^{\otimes K-1},(Q \otimes P)^{\otimes K-1}) \to 1$, where $\text{TV}(P,Q) \triangleq \int |\diff P-\diff Q|/2$ denotes the total variation distance.
Using $D(P\|Q) \geq \log \frac{1}{2(1-\text{TV}(P,Q))}$ \cite[(2.25)]{Tsybakov09} and the tensorization property of KL divergence for product distributions,
we have $(K-1)(D(P\|Q)+D(Q\|P)) \to \infty$. By the assumption \prettyref{eq:divquadassump2} and the fact that $E_Q(\theta)$ is non-decreasing in $\theta\in[-D(Q\|P), D(P\|Q)]$,
it follows that
\begin{align*}
D(P\| Q) = E_Q ( D(P\|Q) ) \ge E_Q ( -D(Q\|P )/2 ) \ge \frac{c}{4} D(Q \|P).
\end{align*}
Hence, we have $(K-1) D(P\|Q) \to \infty$, which implies $KD(P\|Q) \to \infty$.

Next we show the second condition in  \prettyref{eq:weak-bdd_nec}  is necessary.
Let $H(X) $ denote the entropy function of a discrete random variable $X$ and $I(X ; Y)$ denote the mutual information
between random variables $X$ and $Y$.
Let $\xi=(\xi_1,\ldots,\xi_n)$ be uniformly drawn from the set $\{x \in \{0,1\}^n: w(x) = K\}$ where $w(x) = \sum x_i$ denotes the Hamming weight; therefore $\xi_i$'s are 
individually $\Bern(K/n)$.
Let $\Expect[d_H(\xi, \hat \xi)] = \epsilon_n K$, where $\epsilon_n\to 0$ by assumption. Consider the following chain of inequalities,
which lower bounds the amount of information required for a distortion level $\epsilon_n$:
\begin{align*}
I(A; \xi) & \overset{(a)}{\geq} I(\hat \xi; \xi)  \geq \min_{\Expect[d(\tilde\xi,\xi)] \leq \epsilon_n K} I(\tilde \xi; \xi)
 \geq H(\xi)- \max_{\Expect[d(\tilde\xi,\xi)] \leq \epsilon_n K}  H( \tilde\xi \oplus \xi ) \\
& \overset{(b)}{=} \log \binom{n}{K} - n h\pth{\frac{\epsilon_n K}{n}} \overset{(c)} {\ge } K \log \frac{n}{K} (1+o(1)) ,
\end{align*}
 where $(a)$ follows from the data processing inequality, $(b)$ is due to the fact that\footnote{To see this, simply note that $H(X) \leq  \sum_{i=1}^n H(X_i) \leq n h(\sum \prob{X_i=1}/n) \leq n h(p)$ by Jensen's inequality, which is attained with equality
 when $X_i$'s are iid $\Bern(p)$.} $\max_{\Expect[w(X)] \leq p n } H(X) = nh(p)$ for any $p \leq 1/2$ where $h(p) \triangleq p \log \frac{1}{p} + (1-p) \log \frac{1}{1-p}$ is the binary entropy function,
and $(c)$ follows from the  bound $\binom{n}{K}  \geq \pth{ \frac{n}{k} }^K$, the assumption $K/n$ is bounded away from one, and the bound $h(p) \leq -p\log p + p$ for $p\in [0,1]$.
 Moreover,
\begin{align}
I(A; \xi)
= & ~ \min_\mathbb{Q} D({ \mathbb{P}_{A|\xi} } \| \mathbb{Q} |\mathbb{P}_{\xi})	\nonumber \\
\leq & ~ 	  D({ \mathbb{P}_{A|\xi} } \|Q^{\otimes \binom{n}{2}} |\mathbb{P}_{\xi})  \nonumber \\
= & ~ 	\binom{K}{2} D(P \| Q).  \label{eq:K2D_bnd}
\end{align}
Combining the last two displays, we get that $ \liminf_{n\to\infty} \frac{(K-1) D(P\|Q)}{\log (n/K)} \ge 2$.

\begin{remark}   \label{rmk:nonzero_diagonal}
The hidden community model (\prettyref{def:model}) adopted in this paper
assumes the data matrix $A$ has zero diagonal, meaning that we observe no self information about the individual vertices -- only pairwise information.   A different assumption
used in the literature for the Gaussian submatrix localization problem is that
$A_{ii}$ has distribution $P$ if $i \in C^*$ and distribution $Q$ otherwise.
 \prettyref{thm:weak_general} holds for that case with the
modification  that the factors $K-1$  in \prettyref{eq:weak-bdd_suff} and \prettyref{eq:weak-bdd_nec}
are replaced by $K+1$.      We explain briefly why the modified theorem is true.  
The proof for the sufficient part goes through with the definition of $e(S,T)$ in \prettyref{eq:eST} modified to include
diagonal terms indexed by $S\cap T$:
 $e(S,T) = \sum_{(i\leq j): (i,j) \in (S\times T) \cup (T\times S)} L_{ij}$. 
 Then $m$ increases by $K-\ell$, resulting in $K-1$ replaced by $K+1$ in
 $E_1$ and $E_2$.
As for the necessary conditions, the proof of the
first part of \prettyref{eq:weak-bdd_nec} goes through with the sufficient statistic $T$
extended to include  two more variables, $A_{ii}$ and $A_{JJ},$  which has the effect of
increasing $K$ by one, so the first part of \prettyref{eq:weak-bdd_nec}  holds with
$K$ replaced by $K+1,$  but the first part of \prettyref{eq:weak-bdd_nec}  has the same
meaning whether or not $K$ is replaced  by $K+1$.
The proof of the second part of \prettyref{eq:weak-bdd_nec} goes
through with $\binom{K}{2}$ replaced by $1+ \dots + K = \binom{K+1}{2}$ in \prettyref{eq:K2D_bnd},
which has the effect of changing $K-1$ to $K+1$ in the second part
of \prettyref{eq:weak-bdd_nec}.
The necessary conditions and the sufficient conditions for exact
recovery stated in the next section hold without modification for the
model with diagonal elements.   In the proof of \prettyref{lmm:exact_nec_general},
the term $e(i,C^*)$ in the definition of $F$, \prettyref{eq:Fdef},  should include the
term $L_{ii}$ and
the random variable $X_i$ in the proof that $\prob{E_1}\to 0$ should
be changed to $X_i=e(i, \{1, \cdots , i\}),$  and also include the term $L_{ii}$.
\end{remark}

\subsection{A Sufficient Condition For Weak Recovery With Random Cluster Size} \label{sec:suff_weak_random}
\prettyref{thm:weak_general} invokes the assumption that $|C^*| \equiv K$ and $K$ is known.
 In the proof of exact recovery, as we will see,
we need to deal with the case where $|C^*|$ is random and unknown.  For that reason,
the following lemma gives a sufficient condition for weak recovery with a random cluster size.
We shall continue to use $\widehat{C}_{\rm ML}$ to denote the estimator defined by
\prettyref{eq:MLE}, although in this context it is not actually the MLE  because $|C^*|$ need not be $K$.     
That is, there is a (slight) mismatch between the problem the estimator was designed for and the problem it is applied to.
\begin{lemma}[Sufficient condition for weak recovery with random cluster size]\label{lmm:weak_general_random_suff}
Assume that $ K \to \infty,$ $\limsup K/n < 1,$ and there exists a universal constant $C>0$ such that \prettyref{eq:divquadassump} holds.
Furthermore, suppose that
\begin{align*}
\prob{  \big| |C^*| - K \big| \leq   K/ \log K  }  \geq 1-o(1).
\end{align*}
If \prettyref{eq:weak-bdd_suff} holds, then
 $$
\Prob\left\{ |\widehat{C}_{\rm ML} \triangle C^*| \le 2 K \epsilon  + 3 K/\log K  \right\} \geq 1  -o(1),
$$
where $\epsilon= 1/ \sqrt{ \min\{ \log K, K D(P\|Q)  \} }$.
\end{lemma}
\begin{proof}
By assumption, with probability converging to $1$,
$\big| |C^\ast| - K \big| \leq K/\log K$.
In the following, we assume that $|C^\ast|=K'$ for $|K' -K | \leq K/\log K$.
Let $L=|\widehat{C}_{\rm ML} \cap C^*|$.
Then $|\widehat{C}_{\rm ML} \triangle C^*| = K+K'-2L$.
To prove the theorem, it suffices to show that $\Prob\{L \leq (1-\epsilon) K - |K' -K | \} = o(1),$  where $\epsilon$ is defined
in the statement of the theorem.
Following the proof of \prettyref{thm:weak_general} in the fixed cluster size case, we get that for all $ 0 \le \ell \le K-1$,
\begin{align*}
  \{ L = \ell \}  \subset &~  \{  \exists C \subset [n]: |C| = K, |C\cap C^*| = \ell, e(C, C) \geq e(C^*,C^*) \} \nonumber \\
= &~   \{  \exists S \subset C^*, T \subset (C^*)^c: |S| = K'-\ell, |T|=K-\ell,    e(S, C^*)                          \le e(T, T) +  e(T, C^*\backslash S)   \} \nonumber \\
\subset &~   \{  \exists S \subset C^*: |S| =K'-\ell,   e(S, C^*)             \leq m \theta  \} \nonumber \\
&~\cup \{ \exists S \subset C^*, T \subset (C^*)^c: |S| = K'-\ell, |T|=K-\ell, e(T,  T) +  e(T, C^\ast \backslash S) \ge m \theta \},
\end{align*}
where $\theta \in [-D(Q\|P), D(P\|Q)]$ is chosen later. Notice that $e(S,C^*)$ has the same distribution as $ \sum_{i=1}^{m'} L_i$ under measure $P$;
$e(T,  T) + e(T, C^\ast \backslash S)$ has the same distribution as $ \sum_{i=1}^m L_i$ under measure $Q$
where $m'=\binom{K'}{2}-\binom{\ell}{2}$, $m=\binom{K}{2}-\binom{\ell}{2}$,
and $L_i$ are i.i.d.\ copies of $\log \frac{ \diff P}{ \diff Q} $.
Hence, by the union bound and large deviation bounds in \prettyref{eq:LDupper1} and \prettyref{eq:LDupper2},
\begin{align*}
\prob{L=\ell} & \le  \binom{K'}{K'-\ell} P \qth{ \sum_{i=1}^{m'} L_i \le m \theta } +
\binom{n-K'}{K-\ell} \binom{K'}{K'-\ell} Q \qth{\sum_{i=1}^{m} L_i \ge m \theta } \\
& \le  \left( \frac{K'\eexp }{K'-\ell}  \right)^{K'-\ell}   \eexp^{-m' E_P( m\theta/m') ) } +
\left(  \frac{(n-K')  \eexp }{ K-\ell  } \right)^{K-\ell}    \left(   \frac{K' \eexp }{ K'-\ell } \right)^{K'-\ell} \eexp^{-m E_Q(\theta) }.
\end{align*}
Notice that for any $\ell \le (1-\epsilon)K - |K-K' |$,
$K'-\ell \ge \epsilon \max\{ K', K\} $, $K-\ell \ge \epsilon K$, and
\begin{align*}
 \frac{K}{K+  K/\log K } \le \frac{ K - \ell }{ K' - \ell}  \le  \frac{ K - \ell }{ K - K/ \log K - \ell}  \le \frac{ K- (1-\epsilon)K } { K  - K/\log K  - (1-\epsilon) K } .
 \end{align*}
Since $\epsilon \ge 1/\sqrt{\log K}$ and $K \to \infty$, it follows that $(K-\ell)/(K'-\ell)=1+o(1)$.   Also,
\begin{align*}
m' &=(K'-\ell)(K'+\ell-1)/2 \ge (K'-\ell)(K'-1)/2  \\
m &=(K-\ell)(K+\ell-1)/2  \ge (K-\ell)(K-1)/2,
\end{align*}
Therefore, $m/m'\to 1,$  and, moreover,
$$
\prob{L=\ell}  \le \eexp^{-(K-\ell) (1+o(1)) E_1} + \eexp^{-(K-\ell)  (1+o(1)) E_2},
$$
with
\begin{align*}
E_1 & = K E_P( m\theta/m')/2 - \log \frac{ \eexp}{ \epsilon},   \\
E_2 & =  K  E_Q(\theta)/2   - \log \frac{(n-K') \eexp^2}{K \epsilon^2} .
\end{align*}
By the assumption \prettyref{eq:weak-bdd_suff}, we have $K D(P\|Q) (1-\eta) \geq 2  \log \frac{n}{K}$ for some $\eta \in (0,1)$.
Choose $\theta = (1-\eta) D(P\|Q)$.
By \prettyref{eq:divquadassump}, we have that
$E_P(\theta) \geq c \eta^2 K D(P\|Q) $ and $E_P( m\theta/m') \geq (1+o(1))  c \eta^2 K D(P\|Q) $.
Thus,
 \begin{align*}
 E_1 \geq  (1+o(1)) c \eta^2 K D(P\|Q)/2 - \log \frac{\eexp}{\epsilon} .
 \end{align*}
 Using the fact that $E_P(\theta)=E_Q(\theta)-\theta$, we get that
\[
E_2 \geq c  \eta^2 K D(P\|Q)/2- 2 \log \frac{\eexp}{\epsilon} + \frac{K}{2} D(P\|Q) (1-\eta) -  \log \frac{n-K'}{K} \geq c K \eta^2 D(P\|Q))/2 - 2 \log \frac{\eexp}{\epsilon}.
\]
Since $K D(P\|Q) \to \infty$ by assumption $\epsilon \ge 1/\sqrt{K D(P\|Q)}$, it follows that $E=\min\{E_1, E_2\}= \Omega(KD(P\|Q))$.
Therefore,\footnote{The $o(1)$ terms converge to zero as $\frac{K}{K'}\to 1$ and $\frac{m}{m'}\to 1,$  uniformly in $\ell$ for $0\leq \ell \leq (1-\epsilon) K- |K-K' | $.}
\begin{align*}
\prob{L \le (1-\epsilon) K- |K' -K |} &  \leq \sum_{\ell=0}^{(1-\epsilon) K }
 \left( \eexp^{-(K-\ell) (1+o(1)) E_1} + \eexp^{-(K-\ell)  (1+o(1)) E_2}  \right) \\
& \leq 2 \sum_{\ell=\epsilon K}^{\infty} \eexp^{- (1+o(1)) \ell E} = \exp( - \Omega( \sqrt{ K^3 D(P\|Q) } )  ) = o(1),
\end{align*}
as was to be proved.
\end{proof}

\section{Exact Recovery for General $P/Q$ Model}\label{sec:exactgeneral}

The sufficiency and necessity halves of \prettyref{thm:exact_general} are proved in
Sections \ref{sec:suff_exact} and \ref{sec:nec_exact}, respectively.

\subsection{The Sufficient Condition and the Voting Procedure} \label{sec:suff_exact}

This section proves the sufficiency part of \prettyref{thm:exact_general}.
The proof is based on a two-step procedure for exact recovery, described
as \prettyref{alg:PQ_cleanup}.    The first main step of the algorithm
 (approximate recovery) uses
an estimator capable of weak recovery, even with a slight mismatch
between $|C^*|$ and $K$, such as provided by the ML estimator
(see \prettyref{lmm:weak_general_random_suff}). The second main step cleans
up the residual errors through a local voting procedure for each index. In order
to make sure the first and second step are independent of each other, we use
the method of successive withholding.

This method of proof highlights
 $\eqref{eq:voting-suff}$ as the sufficient condition for when the local
 voting procedure  succeeds.  In fact, it permits us to prove an intermediate result, 
 \prettyref{thm:metaweakplusvotingX} below, which can be used to show
 that weak recovery plus cleanup in linear additional time can be applied to
 yield exact recovery no matter how the weak recovery step is achieved.  In particular, 
\cite{HajekWuXu_one_beyond_spectral15} and \cite{HajekWuXu_MP_submat15}
give conditions for message passing algorithms to achieve weak recovery
in (near linear) polynomial time, and they invoke \prettyref{thm:metaweakplusvotingX}
to note that, if  \eqref{eq:voting-suff} holds, exact recovery can be achieved
with the addition of the linear time cleanup step.

\begin{algorithm}
\caption{Weak recovery plus cleanup for exact recovery}\label{alg:PQ_cleanup}
\begin{algorithmic}[1]
\STATE Input: $n \in \naturals$, $K >0$, distributions $P$, $Q$; observed matrix $A;$
$\delta \in (0,1)$ with $1/\delta, n\delta \in \naturals$.
\STATE (Partition): Partition $[n]$ into $ 1/\delta$ subsets $S_k$ of size $n\delta$.
\STATE (Approximate Recovery) For each $k=1, \ldots,  1/\delta $, let $A_k$ denote the restriction of $A$ to the rows and columns with index
in $[n]\backslash S_k$, run an estimator capable of weak recovery with input $(n(1-\delta), \lceil K(1-\delta)\rceil , P, Q , A_k)$
and let $\hat{C}_k$ denote the output.
\STATE (Cleanup) For each $k=1, \ldots,  1/\delta $  compute $r_i=\sum_{j \in \hat{C}_k } L_{ij}$ for all $i \in S_k$ and return
$\check{C}$, the set of $K$ indices in $[n]$ with the largest values of $r_i$.
\end{algorithmic}
\end{algorithm}

The following theorem gives sufficient conditions under which the two-step procedure achieves exact recovery,
assuming the first step provides weak recovery.

\begin{theorem}\label{thm:metaweakplusvotingX}
Suppose $\tilde{C}$ is produced by \prettyref{alg:PQ_cleanup} using estimators for weak recovery $\hat C_k$ such that,
\begin{align}
\prob{  | \hat{C}_k \Delta C_k^\ast | \le \delta K \text{ for } 1 \le k \le 1/\delta } \to 1,  \label{eq:weakrecoverycondition}
\end{align}
as $n\to\infty,$  where   $C_k^\ast= C^\ast \cap ( [n] \backslash S_k )$.
Suppose also that \prettyref{ass:reg_strong} holds (or the weaker conditions \eqref{eq:subgQ-loc} and \eqref{eq:divquadassump} hold),
\eqref{eq:voting-suff}  holds.  Then $\pprob{ \check{C}=C^\ast } \to 1$ as $n\to \infty$.
\end{theorem}

The proof of \prettyref{thm:metaweakplusvotingX} is given after the following lemma.
\begin{lemma}  \label{lmm:sum_LDP}
Suppose \prettyref{ass:reg_strong} holds  (or the weaker condition \eqref{eq:subgQ-loc} holds)  and  \eqref{eq:voting-suff} holds.
Let $\{X_i\}$ denote a sequence of i.i.d.\ copies of $\log \frac{ \diff P}{ \diff Q} $ under measure $P$.
Let $\{Y_i\}$ denote another sequence of i.i.d\  copies of $ \log \frac{ \diff P}{ \diff Q}$ under measure $Q$,
which is independent of $\{X_i\}$.   Then for $\delta$ sufficiently small and $\gamma = \frac{1}{K} \log \frac{n}{K},$\footnote{The $o$ in
$o(1/K)$ is understood to hold as $n\to\infty.$   Thus, if $K$ is bounded, $o(1/K)$ means $o(1)$ as $n\to \infty.$}
\begin{eqnarray}
\prob{ \sum_{i=1}^{K(1-2\delta)}  X_i + \sum_{i=1}^{K \delta} Y_i \le K(1-\delta) \gamma}  &= & o(1/K)   \label{eq:localvoting1} \\
 \prob{ \sum_{i=1}^{K(1-\delta)}  Y_i \ge  K(1-\delta) \gamma } &=&  o(1/(n-K)). \label{eq:localvoting2}
\end{eqnarray}
\end{lemma}

\begin{proof}
By the assumption \eqref{eq:voting-suff}, there exists $\epsilon > 0$ sufficiently
small such that  $K E_Q(\gamma) \geq (1+\epsilon) \log n $ for all sufficiently large $n$.   We restrict attention to such $n$.
First of all,
\[
 \prob{ \sum_{i=1}^{K(1-\delta)}  Y_i \ge  K(1-\delta) \gamma  } \leq \exp( - K(1-\delta) E_Q(\gamma)) \leq n^{ - (1-\delta)(1+\epsilon)}.
\]
Then \prettyref{eq:localvoting2} holds as long as $\delta <  \frac{\epsilon}{1+\epsilon}$. To show \eqref{eq:localvoting1},  
for any $t > 0$, the Chernoff bound yields
\begin{align*}
	\prob{ \sum_{i=1}^{K(1-2\delta)}  X_i + \sum_{i=1}^{K \delta} Y_i \le K(1-\delta) \gamma   }   
\leq \exp\pth{K(1-2\delta) (\psi_P(-t)+\gamma t) + K\delta (\psi_Q(-t)+t\gamma)}.
\end{align*}
Since $E_P(\gamma) = \sup_{-1 \leq \lambda \leq 0} \lambda \gamma - \psi_P(\lambda)$, choose $t \in [0,1]$ so that $\psi_P(-t)+\gamma t = - E_P(\gamma) = - E_Q(\gamma)+\gamma$. 
Since $\lambda \mapsto \psi_Q(\lambda)$ is convex with $\psi_Q(0)=\psi_Q(1)=0,$ it follows
that
\begin{align}
\psi_Q(-t)  \leq   \psi_Q(-1)  \leq D(Q\|P) \left( 1+ C/2 \right), 
\label{eq:chiC}
\end{align}
where the last inequality follows from  \prettyref{eq:subgQ-loc} with $\lambda =-1$.  
Note that \prettyref{eq:divquadassump2} is implied by \prettyref{eq:subgQ-loc}.
It follows from \prettyref{eq:divquadassump2} that $E_Q(\gamma) \ge E_Q(0) \ge \frac{1}{2C} D(Q\|P)$.
Together with \prettyref{eq:chiC}, it yields that 
$\psi_Q(-t) \leq  C(C+2) E_Q(\gamma)$. Let $C'=C(C+2)$. 
Combining the above gives 
\begin{align*}
	 \prob{ \sum_{i=1}^{K(1-2\delta)}  X_i + \sum_{i=1}^{K \delta} Y_i \le K(1-\delta) \gamma  }   
\leq & ~\exp\pth{-K (1-2\delta) E_P(\gamma)  + K \delta C' E_Q(\gamma) + K \delta \gamma } \\
= & ~\exp\pth{-K (1-(C'+2)\delta) E_P(\gamma)  + K \delta (1+C') \gamma } \\
\leq & ~\exp\pth{-(1-(C'+2)\delta) (\log K + \epsilon \log n)  + \delta (1+C') \log n },
\end{align*}
where the last inequality follows from the assumption that $KE_P(\gamma) = \log K - \log n  + K E_Q(\gamma) \geq \log K + \epsilon \log n$.
Therefore, as long as $(1-(C'+2)\delta)(1+\epsilon/2) > 1$ and $\delta (1+C') \leq (\epsilon/3)/(1+\epsilon/2)$, 
$$
\prob{ \sum_{i=1}^{K(1-2\delta)}  X_i + \sum_{i=1}^{K \delta} Y_i \le K(1-\delta) \gamma}   
 \leq \exp\left( - \left( \frac{1}{1+\epsilon/2}\right)    \left( \log K + \frac{2\epsilon}{3} \log n \right)   \right),
$$
so that \prettyref{eq:localvoting1} holds.
\end{proof}

\begin{proof}[Proof of \prettyref{thm:metaweakplusvotingX}]
Note that the conditions of \prettyref{lmm:sum_LDP} are satisfied, so that
\prettyref{eq:localvoting1} and \prettyref{eq:localvoting2} hold.

Given  $(C_k^*,  \hat{C}_k)$,  each of the random variables $r_i \in S_k $ for $i \in [n]$
is conditionally the sum of  independent random variables, each with either
the distribution of $X_1$ or the distribution of $Y_1$ described in \prettyref{lmm:sum_LDP}.
Furthermore, on the event,
${\cal E}_k =\{    | \hat{C}_k   \triangle C_k^*    |   \leq  \delta K  \},$
$$
 |\hat{C}_k  \cap C_k^*  |  \geq  |\hat{C}_k| -  |\hat{C}_k   \triangle C^*_k | =
  \lceil K(1-\delta)\rceil -     |\hat{C}_k   \triangle C^*_k |      \geq K(1-2\delta),
$$
One can check by definition and the change of measure that $X_1$ is first-order stochastically greater than or equal to $Y_1$.
Therefore, on the event ${\cal E}_k ,$
 for $i \in C^*$,  $r_i $ is stochastically greater than or equal to $ \sum_{j=1}^{K(1-2\delta)}  X_j + \sum_{j=1}^{K \delta} Y_j$.
 For $i \in [n] \backslash C^*$, $r_i$ has the same distribution as $\sum_{j=1}^{K(1-\delta)}  Y_j$.
Hence, by \prettyref{eq:localvoting1} and \prettyref{eq:localvoting2} and the union bound,
with probability converging to $1$,  $r_i >K(1-\delta)  \gamma $ for all $i \in C^*$ and $r_i <K(1-\delta) \gamma $ for all $i \in [n] \backslash C^*$.
Therefore, $\pprob{ \check{C}=C^\ast } \to 1$ as $n\to \infty$.
\end{proof}

\begin{proof}[Proof of Sufficiency Part of \prettyref{thm:exact_general}]
If $K$ is bounded, exact recovery is the same as weak recovery, so the sufficiency part of \prettyref{thm:exact_general}
follows from the sufficiency part of \prettyref{thm:weak_general} in that case.   So assume for the remainder of
the proof that $K\to \infty.$

 In view of \prettyref{thm:metaweakplusvotingX} it suffices to verify \prettyref{eq:weakrecoverycondition} when $\hat C_k$
for each $k$ is the MLE for $C^{*}_k$ based on observation of $A_k,$   for $\delta$ sufficiently small.
The distribution of $|C^{*}_k|$ is obtained by sampling the
indices of the original graph without replacement.
 Therefore, by a result of Hoeffding \cite{Hoeffding63}, the distribution of
$|C^{*}_k|$ is convex order dominated by the distribution that would result by sampling with replacement, namely, by
$\Binom\left(n(1-\delta), \frac{K}{n}\right)$.
That is, for any convex function $\Psi,$  $\expect{ \Psi( |C^{*}_k| )} \leq \expect{ \Psi(  \Binom(n(1-\delta), \frac{K}{n}) )  }$.
Therefore, Chernoff bounds for  $\Binom(n(1-\delta), \frac{K}{n}) ) $ also hold for $|C^{*}_k|$.
The Chernoff bounds for $X \sim \Binom(n,p)$ give:
\begin{align}
\prob{ X \ge (1+\eta) n p } \le \eexp^{-\eta^2 n p /3}, \quad \forall ~0\le \eta \le 1 \label{eq:BinomChernoffQuadratic1}\\
\prob{ X \le (1-\eta) n p } \le \eexp^{-\eta^2 n p /2}, \quad \forall  ~0\le \eta \le 1\label{eq:BinomChernoffQuadratic2}.
\end{align}
Then,
\begin{align*}
\prob{  \big| | C_k^\ast| -(1-\delta)K \big| \ge  \frac{K}{\log K} }  & \le
  \prob{ \bigg| \Binom\left(n(1-\delta), \frac{K}{n}\right)  -(1-\delta)K \bigg| \ge   \frac{K}{\log K}  } \\
& \le \eexp^{-\Omega(K/\log^2 K) }=o(1).
\end{align*}
Since  \prettyref{eq:weak-bdd_suff} holds and $K\to \infty$, it follows that
$$
 \liminf_{n\to\infty} \frac{\lceil  (1-\delta)K\rceil  D(P\|Q)}{\log \frac{n }{K}}> 2
 $$
for any sufficiently small $\delta \in (0,1)$ with $1/\delta, n\delta \in \naturals$.
Hence, we can apply \prettyref{lmm:weak_general_random_suff} with $K$ replaced by
$\lceil  (1-\delta)K\rceil $ to get that for any $1 \le k\le 1/\delta$,
\begin{align}
\prob{  | \hat{C}_k \Delta C_k^\ast | \le  2 \epsilon K+ 3 K/\log K } \geq 1-o(1),
\end{align}
where $\epsilon=1/\sqrt{\min\{\log K, K D(P\|Q) \}}$.
Since $\delta$ is a fixed constant, by the union bound over all $1 \le k\le 1/\delta$, we have that
\begin{align*}
\prob{  | \hat{C}_k \Delta C_k^\ast | \le 2\epsilon K+  3 K/\log K  \text{ for } 1 \le k \le 1/\delta } \geq 1-o(1).
\end{align*}
Since $\epsilon \to 0$,  the desired
\prettyref{eq:weakrecoverycondition} holds.
\end{proof}

\subsection{The Necessary Condition} \label{sec:nec_exact}

The following lemma gives a necessary condition for exact recovery under the general $P/Q$ model
expressed in terms of probabilities for certain large deviations.   Later in the section the lemma
is combined with the large deviations lower bound of \prettyref{lmm:ld-lb}
to establish the necessary conditions in \prettyref{thm:exact_general}.   This method
parallels the method used in the previous section for establishing the sufficient condition
in  \prettyref{thm:exact_general}.

\begin{lemma}\label{lmm:exact_nec_general}
Assume that $K \to \infty$ and $\limsup K/n < 1$.
Let $L_i$ denote i.i.d.\ copies of $\log \frac{dP}{dQ}$.
If there exists an estimator $\widehat{C}$ such that $\pprob{ \widehat{C} = C^\ast } \to 1$,
then for any $K_o \to \infty$ such that $K_o=o(K)$,
there exists a threshold $\theta_n$ depending on $n$ such that for all sufficiently large n,
\begin{align}
  P \qth{ \sum_{i=1}^{K-K_o}  L_i  \le (K-1) \theta_n  - (K_o-1) D(P \|Q)  -  6\sigma  } & \le \frac{2}{K_o}, \label{eq:votingtype2error} \\
 Q \qth{ \sum_{i=1}^{K-1} L_i\ge (K-1) \theta_n } & \le \frac{1}{n-K}, \label{eq:votingtype1error}
\end{align}
where $\sigma^2=K_o \var_P\left( L_1 \right) $ and $\var_P(L_1)$ denotes the variance of $L_1$ under measure $P$.
\end{lemma}
\begin{proof}
Since the planted cluster $C^*$ is uniformly distributed, the MLE minimizes the
error probability among all estimators. Thus, without loss of generality, we can assume
the estimator used $\widehat{C}$ is $\widehat{C}_{\rm ML}$ and the indices are numbered
so that $C^*=[K]$.
Hence, by assumption,
$\prob{ \widehat{C}_{\rm ML} = C^\ast } \to 1$.
For each $i\in C^*$ and $j \notin C^*$, we have
$$
e\left(C^*\backslash\{i\}\cup\{j\},C^*\backslash\{i\}\cup\{j\} \right) - e(C^*,C^*) = e(j,C^\ast \backslash \{i \}  ) -  e(i,C^\ast) 
$$
Let $i_0$ denote the random index such that $i_0=\arg \min_{i \in C^*} e(i, C^*)$. 
Let $F$ denote the event that
\begin{align}  \label{eq:Fdef}
\min_{i \in C^*} e(i, C^*)  \le \max_{j \notin C^*} e(j, C^\ast \backslash \{i_0\} ) ,
\end{align}
which implies the existence of $j \notin C^\ast $, such that the set $C^*\backslash\{i_0\}\cup\{j\}$ achieves a likelihood at least as large as that achieved by $C^*$. Since if the event $F$ happens, then with probability at least $1/2$,
ML estimator fails, it follows that $ \frac{1}{2} \prob{F}   \le \prob{\text{\ML fails} }=o(1)$.

Set $\theta'_n$ to be
\begin{align*}
 \theta'_n = \inf \left \{ x \in \reals:   P \qth{ \sum_{i=1}^{K-K_o}  L_i \le (K-1) x - (K_o-1) D(P \|Q) -  6\sigma  }  \ge \frac{2}{K_o} \right \},
\end{align*}
and $\theta_n^{''}$ to be
\begin{align*}
\theta_n^{''} = \sup \left \{ x \in \reals: Q \qth{ \sum_{i=1}^{K-1} L_i\ge (K-1) x }  \ge \frac{1}{n-K} \right \}.
\end{align*}
Define the events 
\begin{align*}
E_1 = \Big\{\min_{i \in C^*} e(i, C^*)  \le (K-1) \theta'_n\Big\}, \quad E_2 = \Big\{\max_{j \notin C^*} e(j, C^\ast \backslash \{i_0\} ) \ge (K-1) \theta^{''}_n\Big\}.
\end{align*}
We claim that $\prob{E_1}=\Omega(1)$ and $\prob{E_2}=\Omega(1)$; the proof is deferred to the end.
Note that the random index $i_0$ only depends on the the joint distribution of edges with both two endpoints in $C^*$.
Thus  $e(j, C^*\backslash\{i_0\}) $ for different $j \notin C^*$ are independent 
and identically distributed, with the same distribution as $\sum_{i=1}^{K-1} L_i$ under measure $Q$. 
Thus $E_1$ and $E_2$ are independent, so in view of $\prob{F}=o(1)$, 
\begin{align*}
\prob{E_1 \cap E_2  \cap F^c}  \ge \prob{E_1 \cap E_2 }  - \prob{F} = \prob{E_1} \prob{E_2 } - o(1)  =\Omega(1),
\end{align*}

Since
\begin{align*}
E_1 \cap E_2 \cap F^c  \subset \{ \theta'_n > \theta^{''}_n  \},
\end{align*}
and $\theta'_n, \theta^{''}_n$ are deterministic, it follows that
$\theta'_n > \theta''_{n} $ for sufficiently large $n$.
Set $\theta_n= (\theta'_n+ \theta^{''}_n)/2$. Thus $\theta_n < \theta'_n $
and by the definition of $\theta'_n$, \prettyref{eq:votingtype2error} holds.
Similarly, we have that $\theta_n > \theta^{''}_n $ and by the definition of $\theta^{''}_n$,  \prettyref{eq:votingtype1error} holds.

We are left to show  $\prob{E_1}=\Omega(1)$ and $\prob{E_2}=\Omega(1)$.
We first prove that $\prob{ E_2}=\Omega(1)$.
Since $Q \qth{ \sum_{i=1}^{K-1} L_i\ge x }$ is left-continuous in $x$, it follows that $Q \qth{ \sum_{i=1}^{K-1} L_i\ge  (K-1) \theta^{''}_n  }  \ge (n-K)^{-1}$.
Therefore,
\begin{align*}
\prob{E_2} & = 1- \prod_{j\notin C^*} \prob{e(j, C^\ast) < (K-1) \theta^{''}_n   }  \\
&= 1- \left( 1- Q \qth{ \sum_{i=1}^{K-1} L_i\ge  (K-1) \theta^{''}_n }  \right)^{n-K} \\
& \ge 1- \exp \left(  -    Q \qth{ \sum_{i=1}^{K-1} L_i\ge  (K-1)  \theta^{''}_n  }    (n-K)   \right) \ge 1-\eexp^{-1},
\end{align*}
where the first equality holds because $e(j, C^*\backslash\{i_0\} )$ are independent for different $j \notin C^\ast$;
the second equality holds because $e(j, C^*\backslash\{i_0\} )$ has the same distribution as $\sum_{i=1}^{K-1} L_i$ under measure $Q$;
the third inequality is due to $1-x \le \eexp^{-x}$ for $x \in \reals$; the last inequality holds because
$ Q \qth{ \sum_{i=1}^{K-1} L_i\ge  (K-1) \theta^{''}_n  } \ge (n-K)^{-1}$.  So $\prob{ E_2}=\Omega(1)$ is proved.

Next, we show that $\prob{E_1}=\Omega(1)$. The proof is similar to the proof of  $\prob{E_2} =\Omega(1)$ just given, but it is complicated by the fact
the random variables $e(i,C^*)$ for $i\in C^*$ are not independent.
Since  $P \qth{ \sum_{i=1}^{K-K_o} L_i \le x }$ is right-continuous in $x$, it follows from the definition that
\begin{align}
P \qth{ \sum_{i=1}^{K-K_o} L_i \le (K-1) \theta'_n-  (K_o-1) D(P \|Q) -  6\sigma  } \ge \frac{2}{K_o}. \label{eq:KoBinomp}
\end{align}
For all $i \in C^\ast$, $e(i, C^\ast)$ has the same distribution as $\sum_{i=1}^{K-1} L_i$ under measure $P$,
but they are not independent.
Let $T$ be the set of the first $K_o$ indices in $C^\ast$, \ie, $T=[K_o]$, where
$K_o=o(K)$ and $K_o \to \infty$. Let $\sigma^2=K_o \var_P(L_1)$, where $\var_P(L_1)$ denotes the variance of $L_1$ under measure $P$,
and let $T' =\{i\in T  : e(i,T) \leq  (K_o-1) D(P \|Q) + 6\sigma  \}$.    Since\footnote{In case $T'=\emptyset$ we adopt the convention that
the minimum of an empty set of numbers is $+\infty$.}
$$\min_{i\in C^*} e(i,C^*)  \leq  \min_{i \in T'}  e(i,C^*)   \leq   \min_{i\in T'}  e(i,C^*\backslash T) + (K_o-1) D(P \|Q)  + 6\sigma,$$
it follows that
$$
\prob{E_1}\geq \prob{ \min_{j \in T'}  e(j ,C^*\backslash T)   \leq  (K-1) \theta'_n   - (K_o-1) D(P\|Q) - 6 \sigma }.
$$
We show next that  $\prob{|T'|  \geq \frac{K_o}{2} }\to 1$ as $n \to \infty$.
For $i \in T,$   $e(i,T)=X_i+Y_i$ where
$X_i=e(i,\{1, \ldots , i-1\})$ and $Y_i=e(i, \{i+1, \ldots , K_o\})$. The $X$'s are mutually independent, and the $Y$'s are also mutually
independent, and  $X_i $ has the same distribution as $\sum_{j=1}^{i-1} L_j$ and $Y_i $ has the same distribution as
$\sum_{j=1}^{K_o-i} L_j$, where $L_j$ is distributed under measure $P$.
Then $\expect{X_i}= (i-1) D(P\|Q) $ and  $\var(X_i) \leq \sigma^2$.    Thus, by the Chebyshev inequality,
$ \prob{X_i \geq  (i-1) D(P\|Q) + 3\sigma }\leq \frac{1}{9} $
for all $i\in T$.   
Therefore, $|\{i : X_i \leq  (i-1) D(P\|Q) + 3\sigma  \}|$ is stochastically at least as large as a $\Binom\left(K_o, \frac{8}{9}\right)$
random variable,   so that,  $\prob{ |\{i : X_i \leq  (i-1) D(P\|Q) + 3\sigma    \}|  \geq \frac{3K_o}{4} }\to 1$ as $K_o\to \infty$.
Similarly,  $\prob{ | \{i : Y_i \leq  (K_o-i) D(P \|Q)  +3 \sigma  \}|  \geq \frac{3K_o}{4} }\to 1$ as $K_o\to \infty$.
If at least 3/4 of the $X$'s are small and at least 3/4 of the $Y$'s are small, it follows that at least
1/2 of the $e(i,T)$'s for $i \in T$  are small.  Therefore, as claimed,  $\prob{|T'|  \geq \frac{K_o}{2} }\to 1$ as $K_o\to \infty$.

The set $T'$ is independent of  $(e(i,C^*\backslash T): i\in T)$  and each of those variables
has the same distribution as $\sum_{j=1}^{K-K_o} L_j$ under measure $P$.
Thus,
\begin{align*}
& \prob{E_1} \\
\geq  &  1- \expect{ \prod_{j \in T'} \prob{e(j, C^\ast\backslash T ) \geq (K-1) \theta'_n - (K_o-1) D(P\|Q) - 6\sigma }  \bigg|  |T'|\geq \frac{K_o}{2} } - \prob{|T'| < \frac{K_o}{2}}  \\
\geq& 1 - \exp \left(  -    P \qth{ \sum_{j=1}^{K-K_o} L_j  \le  (K-1) \theta'_n- (K_o-1) D(P \|Q) -  6\sigma } K_o/2  \right)  -o(1)  \\
\ge &  1 - \eexp^{-1}  -o(1),
\end{align*}
where the last inequality follows from \prettyref{eq:KoBinomp}.  Therefore, $\prob{ E_1}=\Omega(1)$.
\end{proof}

\begin{proof}[Proof of Necessary Part of \prettyref{thm:exact_general}]
Since the joint condition \eqref{eq:weak-bdd_nec}
is necessary for weak recovery, and hence also for exact recovery,
it suffices to prove \prettyref{eq:voting-nec}
under the assumption that  \eqref{eq:weak-bdd_nec} holds, \ie, 
\begin{equation}
KD(P\|Q) \to \infty, \quad  K D(P\|Q)  \ge (2-\epsilon_0) \log (n/K)	
	\label{eq:weak-nec}
\end{equation}
for any fixed constant $\epsilon_0 \in (0,1)$ and all sufficiently large $n$. 
It follows that 
$$
E_Q\pth{\frac{1}{K} \log \frac{n}{K} } \le E_Q \left( D(P\|Q) \right) = D(P\|Q).
$$
Thus if $K=O(1)$, then \prettyref{eq:weak-nec}
implies \prettyref{eq:voting-nec}. Hence, we assume $K \to \infty$ in the following without loss of generality.

For the sake of argument by contradiction, suppose that \prettyref{eq:voting-nec} 
does not hold.
Then, by going to a subsequence, we can assume that 
\begin{equation}
\limsup_{n\to\infty} \frac{KE_Q(\gamma) }{\log n} <  1,
	\label{eq:exact-nec1}
\end{equation}
where $\gamma = \frac{1}{K} \log \frac{n}{K}$. It follows from \prettyref{eq:weak-nec} that 
$\gamma \le \frac{1}{2-\epsilon_0} D(P||Q)$. 

We shall apply \prettyref{lmm:exact_nec_general} to argue a contradiction. 
As a witness to the nonexistence of $\theta_n$ satisfying \eqref{eq:votingtype2error}  and \eqref{eq:votingtype1error}
we show that if $\theta_n = \gamma$ then neither \eqref{eq:votingtype2error}  nor  \eqref{eq:votingtype1error} holds.
By \prettyref{lmm:equivalence},  $D(P\|Q)\asymp D(Q \|P)$.
Since $0 \leq \gamma \le \frac{1}{2-\epsilon_0} D(P||Q)$,  choosing $\delta>0$ to be a sufficiently
small constant ensures that both $\gamma$ and $ \gamma+\delta D(Q||P)$ lie in $[-D(Q||P), D(P||Q)]$. 
Then \prettyref{ass:reg_strong} and   \prettyref{cor:LD_lower_bnd} yield:
$$
 Q \qth{ \sum_{i=1}^{K-1}  L_i\ge  (K-1) \gamma }  \geq
 \exp\left( - \frac{(K-1)  E_Q(\gamma+\delta D(Q\|P )  ) + \log2}{1- \frac{C }{(K-1)\delta^2  D(Q\|P)  }} \right).
$$
By the properties of $E_Q$ discussed in \prettyref{rmk:CPQ},
$$
E_Q \left(\gamma + \delta D(Q\|P ) \right)  
\leq E_Q(\gamma) +  \delta D(Q\| P),
$$
and by \prettyref{lmm:equivalence},
\begin{align}
\delta D(Q\|P) \le 2  \delta C E_Q(0) \leq 2 \delta C E_Q(\gamma),
\end{align}
so, in view of \eqref{eq:exact-nec1}, if $\delta$ is sufficiently small,
$$
(K-1) E_Q( \gamma+ \delta D(Q\| P) ) < (1- 2\delta) \log n
$$ 
for all sufficiently large $n$. Also,  recall that $D(P\|Q) \asymp D(Q\|P)$ 
 and hence \prettyref{eq:weak-nec} implies that
 $KD(Q\|P) \to \infty$.  
Therefore,
$$
 Q \qth{ \sum_{i=1}^{K-1} L_i\ge  (K-1)\gamma }  \geq  n^{-1+\delta}
 $$
 for all sufficiently large $n$.  Thus, \eqref{eq:votingtype1error} does {\em not} hold for $\theta_n \equiv\gamma$.\\

 Turning to \eqref{eq:votingtype2error} (with $\theta_n= \gamma$), we let $K_o=K/\log K$
 and 
 $$
 \delta' \triangleq \frac{(K_o-1) (D(P\|Q) -\gamma) + 6 \sigma}{ (K-K_o) D(P \|Q)  },
 $$
 where $\sigma= \var_P[L]$.
 Note that $\var_P[L] = \psi_Q''(1)  \le C D(P\|Q)$ by \prettyref{ass:reg_strong} and recall that 
from \prettyref{eq:weak-nec} we have $\gamma \leq \frac{1}{2-\epsilon_0}D(P\|Q)$. 
Furthermore, since $K \diverge$ and $KD(P\|Q) \to \infty$ by \prettyref{eq:weak-nec},
we conclude that $\delta' =o(1)$.

Since $D(P\|Q) \asymp D(Q\|P)$ and $0 \le \gamma \le  \frac{1}{2-\epsilon_0}D(P\|Q)$, 
choosing $\delta$ to be a sufficiently small constant ensures that both $\gamma- \delta' D(P\|Q)$ and $\gamma- ( \delta' + \delta)  D(P \|Q)$ lie in $[-D(Q\|P),D(P\|Q)]$. Hence, applying \prettyref{cor:LD_lower_bnd}  yields
\begin{align}
 & ~ P \qth{ \sum_{i=1}^{K-K_o}  L_i\le (K-1) \gamma  - (K_o-1) D(P\|Q)  - 6 \sigma }  \nonumber \\ 
= & ~ P \qth{ \sum_{i=1}^{K-K_o}  L_i\le (K - K_o) \left(\gamma  - \delta' D(P\|Q) \right) }   \nonumber \\
\geq &~ \exp\left( - \frac{(K-K_o)  E_P\left( \gamma- ( \delta' + \delta)  D(P \|Q) \right) + \log2}{1- \frac{C}{(K-K_o)\delta^2 D(P\|Q) }} \right).
\label{eq:necess_large_dev_lowerbound}
\end{align}
Moreover, in view of the fact that $E_P(\cdot)$ is decreasing and \prettyref{eq:divquadassump},
 \begin{align}
 E_P(\gamma) \ge E_P \left( D(P\|Q)/(2-\epsilon_0) \right) \ge  \frac{ (1-\epsilon_0)^2 D(P\|Q) }{2 (2-\epsilon_0)^2 C} \label{eq:EPgamma}
 \end{align}
 Let $C' = \frac{ (1-\epsilon_0)^2}{2 (2-\epsilon_0)^2 C}$. 
 Therefore, similar to  the properties of $E_Q$ discussed in \prettyref{rmk:CPQ},
\begin{align*}
E_P\left(  \gamma- (\delta' + \delta)  D(P \|Q)     \right) 
&  \leq E_P(\gamma)  + (\delta' + \delta)  D(P \|Q)   \\
&\leq E_P(\gamma) \left( 1+ (\delta'+\delta)/C'  \right) .
\end{align*}
Since $
E_P(\gamma)=E_Q(\gamma)-\gamma,$
by \prettyref{eq:exact-nec1}, there exist some $\epsilon>0$ such that 
$$
K E_P(\gamma)  \le (1-\epsilon) \log n - \log (n/K) =  - \epsilon \log n + \log K \le (1-\epsilon) \log K.
$$
Thus by choosing $\delta$ sufficiently small and in view of $\delta'=o(1)$, 
$$
(K-K_o) E_P\left(  \gamma- (\delta' + \delta)  D(P \|Q)     \right)  \le (1- 2\epsilon') \log K
$$
for some $\epsilon'>0$. Recall that $K D(P\|Q) \to \infty$, it readily follows from \prettyref{eq:necess_large_dev_lowerbound} that 
$$
 P \qth{ \sum_{i=1}^{K-K_o}  L_i\le  (K-1) \gamma  - (K_o-1) D(P\|Q) - 6 \sigma  }  \geq K^{-1+\epsilon'}.
$$
Thus, with $\theta_n =\gamma$, neither \eqref{eq:votingtype2error} nor \eqref{eq:votingtype1error} holds for all sufficiently large $n$.
Therefore, there does not exist a sequence $\theta_n$ such
that both \eqref{eq:votingtype2error} and \eqref{eq:votingtype1error} hold for all sufficiently large $n$, 
contradicting the conclusion of \prettyref{lmm:exact_nec_general}. 
\end{proof}


\begin{appendices}
\section{Equivalence of Weak Recovery in Expectation and in Probability}   \label{app:weak}

\begin{lemma}
There exists an estimator $\hat \xi$ such that  $\frac{d_H(\xi, \hat\xi)}{K} \to 0$ in probability
 if and only if there exists an estimator $\hat \xi$ such that
 $\frac{ \expect{d_H(\xi, \hat\xi)}}{K} \to 0$.
 \end{lemma}

\begin{proof}
One direction is automatic because convergence in $L_1$ implies convergence in probability.
Conversely, suppose $\frac{d_H(\xi, \hat\xi)}{K} \to 0$ in probability for some (sequence of) $\hat \xi$.   Then there exists a deterministic sequence
$\epsilon_n \to 0$  such that  $\pprob{d_H(\xi, \hat\xi) \geq \epsilon_n K} \leq \epsilon_n$.
Define a new estimator by
\[
\tilde \xi = \hat{\xi} \indc{|\hat\xi| \leq K + \epsilon_n K } + \mathbf{0} \cdot \indc{|\hat\xi| > K + \epsilon_n K},
\]
where $\mathbf{0}$ denotes the all-zero vector.
Since $|\xi|=K$, by the triangle inequality, we have
\begin{align*}
\Expect[d_H(\xi, \tilde\xi)] & = \expect{d_H(\xi, \hat \xi) \indc{|\hat\xi| \leq K + \epsilon_n K}} + K \prob{|\hat\xi| > K + \epsilon_n K} \\
&  \leq  \epsilon_n K + \expect{d_H(\xi, \hat \xi) \indc{d_H(\xi,\hat\xi)>\epsilon_n K, \; |\hat\xi| \leq K +  \epsilon_n K}}
   + K \prob{|\hat\xi| > K + \epsilon_n K } \\
&\leq \epsilon_n K + (3K+  \epsilon_n K ) \prob{d_H(\xi,\hat\xi)> \epsilon_n K} \leq 4\epsilon_n K + \epsilon_n^2K.
\end{align*}
Therefore,  $\frac{ \expect{d_H(\xi, \tilde\xi)}}{K} \to 0$.
\end{proof}

\section{Assumption 1 for exponential families of distributions} \label{app:exp}

There is a simple sufficient condition for \prettyref{ass:reg_strong} to hold in case $P$ and $Q$ are from the same exponential
family of distributions (including Bernoulli, Gaussian, etc).   Consider a canonical exponential family with the
following pdf (with respect to some dominating measure):\footnote{For simplicity we assume $T$ and $\theta$
are scaler valued.   Vector values would give  $ p_\theta(x)  = h(x) \exp(\iprod{\theta}{T(x)} - A(\theta)) $ and
the condition \eqref{eq:psipp1}, with $A''(\theta)$ replaced
by $(\theta_1-\theta_0)^\top H(\theta)(\theta_1-\theta_0),$  where $H$ is the Hessian of $A,$ and
$I$ and $J$ becoming line segments, is still sufficient for  \prettyref{ass:reg_strong}. }
\[
p_\theta(x)  = h(x) \exp(\theta T(x)  - A(\theta)),
\]
where $A$ is a convex function.
Then $\Expect_\theta[T] = A'(\theta)$ and $\var_\theta[T] = A''(\theta)$.
Assume that $P$ and $Q$ correspond to parameters $\theta_1$ and $\theta_0,$ respectively.
It could be that $\theta_0 < \theta_1$ or $\theta_1 < \theta_0$; let $I$ denote the interval
with endpoints $\theta_0$ and $\theta_1$ and $J$ denote the interval with endpoints
$\theta_0 \pm (\theta_1 - \theta_0)$.
Then $Q_\lambda$ has parameter $\lambda \theta_1+\bar \lambda \theta_0$.
Furthermore,
\begin{align*}
L = & ~ (\theta_1-\theta_0) T - A(\theta_1) + A(\theta_0) 	\\
D(P\|Q) = & ~ A(\theta_1)  - A(\theta_0) - (\theta_1-\theta_0) A'(\theta_0) 	 \\
C(P,Q) = & ~ - \min_{\theta \in I} A(\theta) \\
\psi_Q(\lambda) = & ~ A(\lambda \theta_1+\bar \lambda \theta_0) - \lambda A(\theta_0) -\bar{\lambda} A(\theta_1) \\
\psi_Q''(\lambda) = & ~ A''(\lambda \theta_1+\bar \lambda \theta_0) (\theta_1-\theta_0)^2.
\end{align*}
By Taylor's theorem,  $D(P\|Q)$ is $ \frac{(\theta_0-\theta_1)^2}{2}$ times a weighted average
of $A''$ over $I$:
$$
D(P\|Q) =   \frac{(\theta_1-\theta_0)^2}{2} \frac{\int_{\theta_0}^{\theta_1} A''(s) (s-\theta_0) ds}{(\theta_1-\theta_0)^2/2}
$$
Similarly, $D(Q\|P)$ is a weighted average of $A''$ over $I$.
Therefore,  a sufficient condition for \prettyref{ass:reg_strong} is
\begin{equation}
	\frac{\max_{\theta \in J} A''(\theta)}{\min_{\theta\in  I  } A''(\theta)} = O(1).
	\label{eq:psipp1}
\end{equation}
Examples:
\begin{enumerate}
\item Gaussian: $\theta=\mu$, $A(\theta)=\theta^2/2$ and $A''(\theta)=1$. So \eqref{ass:reg_strong} holds in the
Gaussian case with no extra assumption.
	\item Bernoulli: $\theta = \log \frac{p}{\bar{p}}$, $A(\theta) = \log(1+e^\theta)$ and $A''(\theta)=\frac{e^\theta}{(1+e^{\theta})^2} = p(1-p)$.
We shall show that if  $p,q$ vary such that $p,q\in (0,1)$ with $p\neq q$,  then \eqref{eq:psipp1} is equivalent to boundedness
of the  LLR.    By symmetry between 0 and 1 we can assume without loss of generality that $0 < q < p < 1$.
First, if $p\leq 1/2$ the LHS of \eqref{eq:psipp1} is  $\frac{p\bar{p}}{q\bar{q}} \asymp \frac{p}{q}$  and if $p\in [1/2,1-\epsilon]$ for some
fixed $\epsilon > 0$ then the LHS of \eqref{eq:psipp1} has size $\Theta(1/q) = \Theta(p/q)$.   So the claim is true if $p$ is bounded away from one.\\
If $p\to 1$ and $q\not\to 1$ then both the LHS of  \eqref{eq:psipp1} and the LLR are unbounded, so the claim is again true. \\
It remains to check the case $p, q\to 1$.  The denominator of the LHS of 
 \eqref{eq:psipp1} is $p\bar{p}\asymp \bar{p}$.    The maximum in the numerator is taken over the
 interval $[\theta_{-1},\theta_1],$  where $\theta_{-1}=\theta_0 - [\theta_1-\theta_0] = \log \left( \frac{q^2\bar{p}}{\bar{q}^2 p}\right)$.
 If $\theta_{-1} \leq 0$ (i.e. $\theta_0 \leq \theta_1/2$) then the numerator of the LHS of  \eqref{eq:psipp1} is $1/4$, so
 \eqref{eq:psipp1} fails to hold, and also,  $\frac{\bar p}{\bar q }=O(\sqrt{\bar p})$ so the LLR is unbounded.    
 It thus remains to consider the case $\theta_1/2 \leq \theta_0 \leq \theta_1$ with $\theta_1 \to \infty$.
 The numerator of the LHS of  \eqref{eq:psipp1} is $r\bar r$ where $r$ is determined by $\theta_{-1}=\log \frac{r}{\bar r},$
 or, equivalently,  $\frac{r}{\bar r} = \frac{q^2\bar{p}}{\bar{q}^2 p}$.
 Hence $\bar r \asymp \frac{(\bar q)^2}{\bar p}$.
 The LHS of  \eqref{eq:psipp1} is $\frac{r\bar{r}}{p\bar{p}} \asymp \frac{\bar r}{\bar p} \asymp \left( \frac{\bar{q}}{\bar{p}}\right)^2$
 while the maximum absolute value of the LLR is $\Theta( \log \frac{\bar{q}}{\bar{p}})$.  Hence, again,  \eqref{eq:psipp1} holds
 if and only if the LLR is bounded.  The claim is proved.
\end{enumerate}

\section{Proof of \prettyref{cor:nec_exact_submat}}  \label{app:Gauss_exact_cor}

In the Gaussian case, $E_Q(\theta) = \frac{1}{8}(\mu+ \frac{2\theta}{\mu})^2$. 
Throughout this proof, let $\theta=\frac{1}{K}\log \frac{n}{K}$ and let $f$
be the function defined by $f(\mu)=E_Q(\theta)= \frac{1}{8}(\mu+ \frac{2\theta}{\mu})^2.$
Consider the equation $f(\mu)=\frac{\log n}{K}.$
It yields a quadratic equation in $\mu^2$: 
$
\mu^4 - \frac{4 \log n + 4 \log K }{K} \mu^2 + \frac{4 \log^2(n/K) }{K^2} =0
$
which has two solutions namely $\mu_{\pm}^2 = \frac{2}{K}  \left( \sqrt{\log n} \pm \sqrt{\log K} \right)^2$.
Without loss of generality, we take  $\mu_+> 0$ and $\mu_ - > 0$; the case of $\mu_+<0$ and $\mu_-<0$
follows analogously. 
In summary, the expressions inside the $\liminf$ in both \eqref {eq:voting-suff} and
\eqref{eq:submat-mle-suff} are one if $\mu$ is replaced by $\mu_+.$

For the sufficiency part, suppose $\mu$ depends on $n$ such that
 \eqref{eq:weak_Gaussian_suff} and  \eqref{eq:submat-mle-suff} hold.
 By  \eqref{eq:submat-mle-suff},
for  $\epsilon > 0$ sufficiently small,   $\mu(1-\epsilon)  \geq  \mu_+$ for all sufficiently
large $n.$   We can also take $\epsilon < 1/10$.
By \eqref{eq:weak_Gaussian_suff},  $\limsup \frac{\theta}{\mu^2} \leq \frac{1}{4}$ so
uniformly for $(1-\epsilon)  \mu  \leq x \leq \mu,$
\begin{align*}
f'(x)  & = \frac{1}{4}\left(x + \frac{2\theta}{x} \right) \left(1-\frac{2\theta}{x^2}\right)   \\
& \geq   \frac{1}{4}\left( (1-\epsilon) \mu  \right) \left(1-\frac{2\theta}{(1-\epsilon)^2 \mu^2}\right)    = \Omega(\mu).
\end{align*}
Also,  $\frac{2 \theta}{\mu_+^2}< 1$ so $f'(x) \geq 0$ for $x \geq \mu_+.$  Hence,
\begin{align*}
\frac{f(\mu)}{f(\mu_+)} - 1 & \geq    \frac{f(\mu) - f(\mu(1-\epsilon))}{f(\mu_+)}  \\
& =   \frac{K}{\log n} \int_{\mu(1-\epsilon)}^{\mu} f'(x) dx  \\
& = \Omega\left( \frac{\epsilon K \mu^2}{\log n} \right)=  \Omega(\epsilon),
\end{align*}
where for the last equality we use $\mu^2 \geq \mu_+^2 \geq \frac{2\log n}{K}.$
Therefore \eqref{eq:voting-suff} holds, sufficiency follows from \prettyref{thm:exact_general}.

For the necessity part, it suffices to show that   \eqref{eq:weak_Gaussian_nec} and
\eqref{eq:voting-nec} imply \eqref{eq:submat-mle-nece}.   If $K \leq n^{1/9}$ then
\eqref{eq:weak_Gaussian_nec} alone implies \eqref{eq:submat-mle-nece}, so we
can also assume that  $K \geq n^{1/9}.$    It follows that
$\frac{2\theta}{\mu_+^2} = \frac{ \sqrt{\log n} - \sqrt{ \log K}  }{ \sqrt{\log n} + \sqrt{ \log K}  } \leq \frac{1}{2}.$
Therefore, for $\epsilon \in (0,0.1),$
\begin{align*}
f(\mu_+ (1-\epsilon) ) &  \leq f(\mu_+)  -  \epsilon \mu_+ \min\{f'(x) : (1-\epsilon)\mu_+ \leq x \leq \mu_+\}   \\
& \leq f(\mu_+) - \frac{\epsilon \mu_+}{4}(1-\epsilon)\mu_+\left( 1 - \frac{1}{2(1-\epsilon)^2}\right)  \\
& \leq f(\mu_+) - \Omega(\epsilon \mu_+^2)  \leq \frac{\log n}{K} (1-\Omega(\epsilon)).
\end{align*}
In view of \eqref{eq:voting-nec} it follows that $\mu \geq \mu_+ (1-\epsilon) $ for all sufficiently large $n.$
Since $\epsilon$ can be arbitrarily small, \eqref{eq:submat-mle-nece} follows.

\end{appendices}

\bibliographystyle{abbrv}
\bibliography{../graphical_combined}

\end{document}